\def\year{2021}\relax
\documentclass[letterpaper]{article} 
\pdfoutput=1
\usepackage{aaai21}  
\usepackage{times}  
\usepackage{helvet} 
\usepackage{courier}  
\usepackage[hyphens]{url}  
\usepackage{graphicx} 
\usepackage{natbib}  
\usepackage{caption} 
\usepackage[utf8]{inputenc} 
\usepackage[T1]{fontenc}    
\usepackage{url}            
\usepackage{booktabs}       
\usepackage{nicefrac}       
\usepackage{microtype}      
\usepackage{enumitem}

\usepackage{mdframed}
\usepackage[switch]{lineno}
\usepackage{amsmath,amsthm,amsfonts,amssymb,amscd}
\usepackage{lastpage}
\usepackage{enumerate}
\usepackage{fancyhdr}
\usepackage{mathrsfs}
\usepackage{xcolor}
\usepackage{graphicx}
\usepackage{listings}
\usepackage{algorithm}
\usepackage{algorithmic}
\usepackage{braket}
\usepackage[margin=0.5cm]{subcaption}

\newcommand{\R}{\mathbb{R}}

\newcommand{\specialcell}[2][c]{%
  \begin{tabular}[#1]{@{}c@{}}#2\end{tabular}}
\DeclareMathOperator*{\argmin}{arg\,min}

\lstdefinestyle{Python}{
    language        = Python,
    frame           = lines, 
    basicstyle      = \footnotesize,
    keywordstyle    = \color{blue},
    stringstyle     = \color{green},
    commentstyle    = \color{red}\ttfamily
}

\newtheorem{problem}{Problem}

\newtheorem{definition}{Definition}

\newtheorem{prop}{Proposition}
\title{Graph Learning for Inverse Landscape Genetics}
\author {

        Prathamesh Dharangutte\textsuperscript{\rm 1},
        Christopher Musco\textsuperscript{\rm 1}\\
}
\affiliations {
    \textsuperscript{\rm 1} Dept. of Computer Science and Engineering, New York University, Brooklyn, New York \\
    prathamesh.d, cmusco@nyu.edu
}

\begin{document}
\maketitle

\begin{abstract}
The problem of inferring unknown graph edges from numerical data at a graph's nodes appears in many forms across machine learning. 
We study a version of this problem that arises in the field of \emph{landscape genetics}, where genetic similarity between organisms living in a heterogeneous landscape is explained by a weighted graph that encodes the ease of dispersal through that landscape.
Our main contribution is an efficient algorithm for \emph{inverse landscape genetics}, which is the task of inferring this graph from measurements of genetic similarity at different locations (graph nodes). 

Inverse landscape genetics is important in discovering impediments to species dispersal that threaten biodiversity and long-term species survival. In particular, it is widely used to study the effects of climate change and human development.
Drawing on influential work that models organism dispersal using graph \emph{effective resistances} \cite{mcrae:2006}, we reduce the inverse landscape genetics problem to that of inferring graph edges from noisy measurements of {these resistances}, which can be obtained from genetic similarity data. 

Building on the NeurIPS 2018 work of \citet{hoskins2018learning} on learning edges in social networks, we develop an efficient first-order optimization method for solving this problem. Despite its non-convex nature, experiments on synthetic and real genetic data establish that our method provides fast and reliable convergence, significantly outperforming existing heuristics used in the field.
By providing researchers with a powerful, general purpose algorithmic tool, we hope our work will have a positive impact on accelerating work on landscape genetics. 
\end{abstract}

\section{Introduction}
Many datasets
can be modeled as a weighted, undirected graph: $G = (V,E)$ with nodes $V = \{v_1, \ldots, v_n\}$ and additional numerical data vectors $x_1,\ldots, x_n \in \R^d$ at each node.
For example, in social networks, each node is a user, each edge is a connection or interaction between users, and $x_i$ might contain demographic information about user $i$ like age, gender, or expressed political party.

Often, node data is correlated with $G$'s \emph{connectivity structure}: if $v_i$ and $v_j$ are strongly connected, $x_i$ and $x_j$ tend to be more similar than for poorly connected nodes \cite{kalofolias2016learn,ortega2018graph}. Formally, connectivity between two nodes can be quantified in many of ways, from simple statistics like shortest path distance or number of common neighbors, to more advanced metrics like personalized PageRank  \cite{page1999pagerank,jeh2003scaling}, SimRank \cite{jeh2002simrank}, or DeepWalk distance \cite{perozzi2014deepwalk}.
While these measures depend solely on $G$'s structure (i.e. edges and their weights), they often align with measured similarities between $x_1, \ldots, x_n$. 

This observation leads to an interesting possibility: even when edges in $G$ are \emph{unknown}, node data can be useful in \emph{inferring edges and weights}, or at least in inferring a graph \emph{whose connectivity structure is consistent with the observed data}. This possibility has been explored across statistics, machine learning, and network science  \cite{Raskutti:2009,cai2011constrained,egilmez2017graph,liben2007link,hoskins2018learning}. In many cases, pairwise measures of connectivity can reveal a striking amount of information about $G$, and by proxy, so can similarity information between $x_1, \ldots, x_n$ \cite{hoskins2018learning}. 
Applications of graph inference from node data include understanding structured statistical correlation, link prediction, and phylogeny reconstruction.

In this work, we examine an application of graph inference in \emph{landscape genetics}, a field at the intersection of landscape ecology, spatial statistics, and population genetics \cite{manel2003landscape,sanderson2020landscape}. Landscape genetics seeks to explain genetic differences between populations of the same species that live at different geographic locations. The goal is to understand how ease of movement between these geographic locations (i.e., through the landscape) affects population genetics. Geographically isolated populations tend to differ genetically, whereas ease of travel and intermixing between populations leads to genetic similarity.

Early methods in landscape genetics correlate genetic similarity with simple measures of geographic isolation, like the Euclidean distance between  populations \cite{wright1943isolation,sokal1978spatial}, or distance along an oriented direction or curve, leading to concepts like clines and ring species \cite{endler1977geographic,huggett2004fundamentals}. These tried-and-true approaches have been successfully applied to understanding genetic variation in a variety of species, including humans \cite{NovembreJohnsonBryc:2008}. More recently, however, work in landscape genetics considers finer-grained measures of landscape-driven isolation, largely based on modeling the landscape as an undirected graph (the \emph{landscape graph}).
Each location (spatial cell) in the landscape is associated with a graph node, and each node is connected by a weighted edge to all geographically adjacent nodes (see Fig. \ref{fig:imageToGridGraph}).
Edge weights are chosen to reflect the ease of organism dispersal between adjacent nodes: we follow the convention that high weight indicates ease of dispersal and low weight indicates inhibition to movement, although note that the opposite meaning is sometimes used  \cite{coulon2004landscape}. Weights are tailored to specific species: e.g., an edge across a span of water would have low weight for a ground-dwelling species which cannot easily traverse the edge. For an organism that prefers low-land environments, edges crossing areas of high elevation might receive lower weight than those crossing low-land areas.

\begin{figure}[hbtp!]
 \begin{subfigure}{0.235\textwidth}
   \centering
   \includegraphics[width=\textwidth]{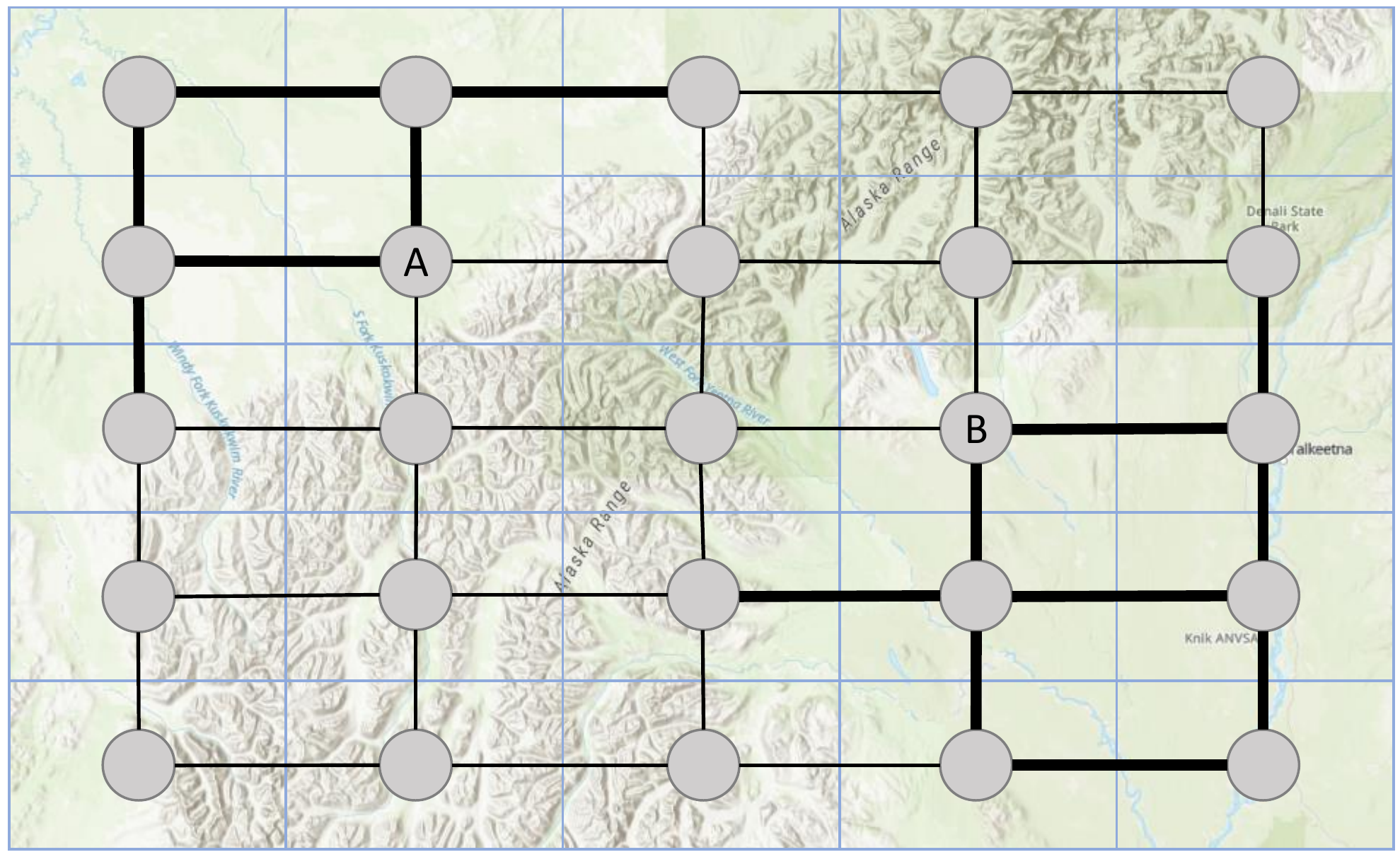}
 \end{subfigure}%
 \hfill
 \begin{subfigure}{0.235\textwidth}
   \centering
   \includegraphics[width=\textwidth]{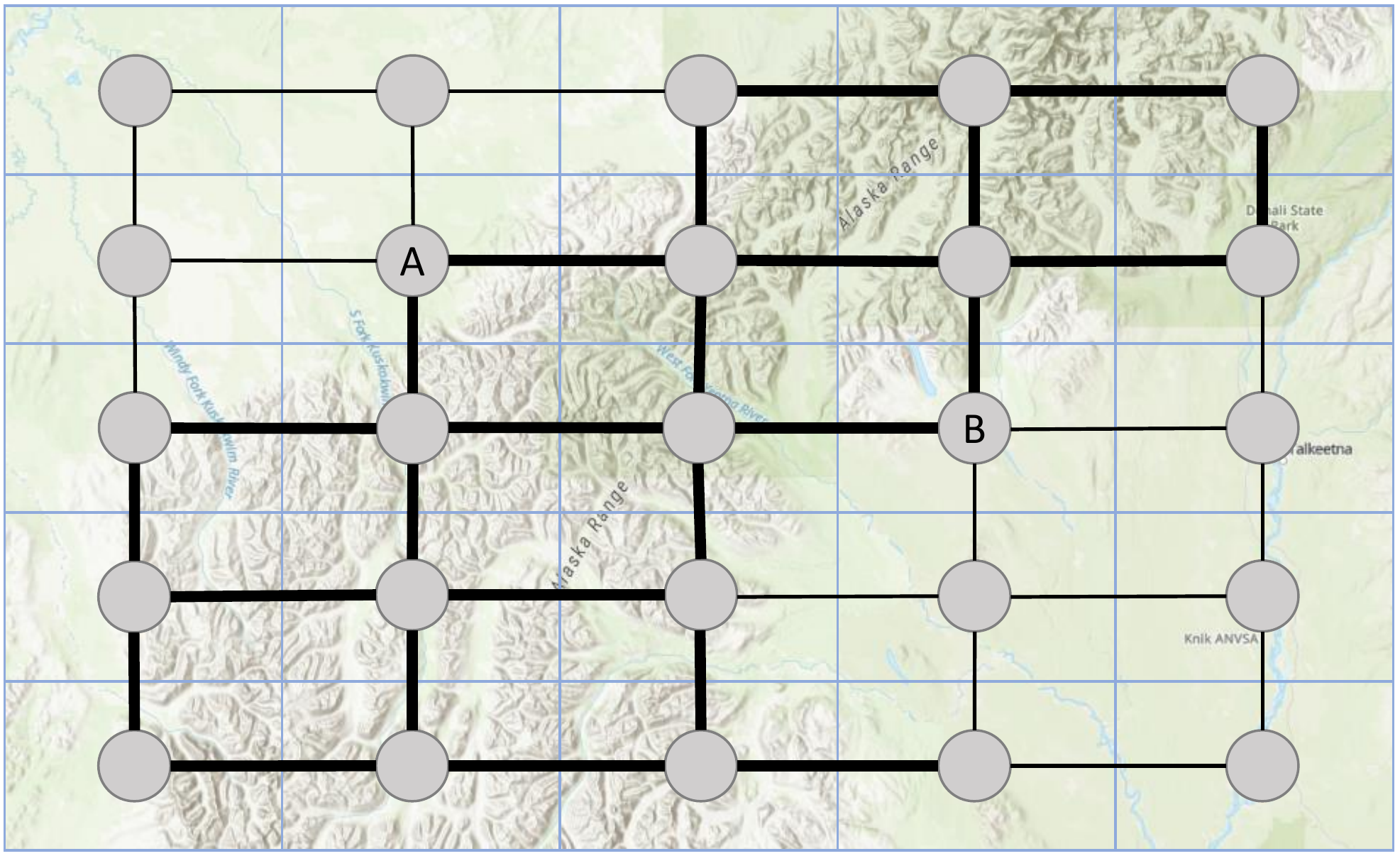}
\end{subfigure}%
\caption{Above we illustrate example landscape graphs for different species, with edge thickness used to indicate weight. In the first, edges across low-altitude areas have higher weight, so this graph would be appropriate for a species that prefers living and traveling in low-land habitat. The second graph, on the other hand, is natural for a species that prefers higher altitude, or mountainous habitat. We expect genetic similarity between two populations living at nodes $A$ and $B$ to be more dissimilar for the first species, since these nodes are less well connected in the first graph.}

\label{fig:imageToGridGraph}
\end{figure}

In addition to the landscape graph, we consider numerical genetic data about populations of organisms living at different nodes of the graph. Usually this data is sparse, meaning we only have information for a subset of nodes \cite{oyler2013sample}. Regardless, the goal is to correlate pairwise genetic similarity between these nodes with pairwise connectivity in the underlying landscape graph. For example, the weight of the {least cost path} between two nodes is a common connectivity measure, and shown to correlate with genetic similarity, measured using e.g., the fixation index \cite{arnaud2003metapopulation,coulon2004landscape,vignieri2005streams}. More recently, McRae's influential paper \emph{Isolation by Resistance} popularized the use of \emph{effective resistance distance} as a connectivity measure in landscape genetics \cite{mcrae:2006,yen2007graph}. Effective resistances better model organism dispersal, and thus correlate more closely with genetic differences across landscapes \cite{mcrae:2007}. 

Amongst many other applications, effective resistance-based landscape ecology has been important in understanding the effects of climate change on species dispersal and migration \cite{circuitscape_white}. \footnote{See \url{https://circuitscape.org/pubs.html} for further details.}

\section{Our Contributions}
\label{sec:contributions}
So where does graph inference come in? Most studies that use landscape graphs to model species dispersal construct these graphs based on \emph{expert knowledge} \cite{mcrae:2006,shirk2010}. Knowledge of a species' behavioral preferences (e.g, preferred elevation, vegetation cover, or climate) are used to determine edge weights, which are then used to compute pairwise connectivities like least cost paths or effective resistances. Multiple landscape graphs proposed by experts can be tested for fit \cite{vos2001genetic,lugon2002phylogeographical,vignieri2005streams,short2011replication}, but achieving high levels of correlation with genetic data requires significant background information on a species (which may be imperfect) and laborious hand-tuning of the landscape graph.

\textbf{Inverse landscape genetics.} To address this issue, there has been interest in moving beyond expert opinion, by \emph{algorithmically} determining optimal edge weights \cite{Zeller:2012,peterman2019comparison}. Specifically, the goal is to learn a function that maps measurable landscape parameters for each edge (e.g. what vegetation cover it goes through, or if there is human development along the edge) to edge weights. The resulting weighted graph should have connectivity structure that correlates as well as possible with genetic differences across the landscape. 

We call this parameterized graph inference problem \emph{inverse landscape genetics}. Not only does this exciting problem offer the possibility of refining expert-designed landscape graphs, but a solution would allow ecologist to infer information about species dispersal based \emph{purely on collected genetic data} \cite{oyler2013sample}, as opposed to the traditional perspective of explaining genetic data with known ecological knowledge. Genetic information could be used to understand species habitat preferences, find bottlenecks in migration, or understand how human development is impeding species movement \cite{circuitscape_white}. As discussed in \citet{Zeller:2012} and \citet{graves2013}, algorithms for learning landscape graphs from data could therefore be essential in future conservation and planning decisions involving e.g. wildlife corridor design.

However, despite interest in the inverse landscape genetics problem, few effective algorithms have been developed to solve it. \citet{Zeller:2012} surveys of existing techniques. Most current approaches optimize landscape graphs (i.e. find a graph consistent with observed genetic data) using variants of brute force search. For example, a common approach is to rely on expert opinion to obtain an initial graph and then search over a small set of nearby weight functions to improve the fit \cite{shirk2010}. 
There has been some work on more systematic algorithms. \citet{peterman2018resistancega} introduce a framework for optimizing landscape graphs using a genetic algorithm and compare their method with other approaches \cite{peterman2019comparison}. \citet{graves2013} develop an approach based on local search heuristics, using Nelder-Mead and Newton line search algorithms to optimize landscape graphs.

\textbf{A differentiable approach.} 
Our main contribution is to show that one of the most common formalizations of the inverse landscape genetics problem can be solved efficiently and reliably using \emph{gradient based} optimization methods. In particular, we consider a version of the problem which correlates the \emph{effective resistance} between two nodes (a measure of graph connectivity) with the \emph{fixation index} between genetic data at those nodes (a measure of genetic differentiation). We build on recent work of \citet{hoskins2018learning} that studies the problem of learning graph edges based on noisy measurements of effective resistances in the graph. As in that result, we show how to compute a gradient for an appropriately chosen graph-learning loss involving the effective resistances, and in our case, fixation index values. To do so, we need to differentiate through the effective resistances computation, which involves the pseudoinverse of a graph Laplacian. We implement this step efficiently using an iterative linear system solver for positive semidefinite matrices. Our approach is detailed in Section \ref{sec:proposed_method}.

To the best of our knowledge, our method is the first for the inverse landscape genetics problem that uses a gradient based optimization method. In Section \ref{sec:experiments}, we compare it against local search heuristics used in prior work \cite{graves2013}, showing that it obtains much more reliable convergence on both synthetic and real-world data sets. As an application of our fast algorithm, we are able to explore questions of statistical complexity that have been raised in the landscape genetics literature \cite{oyler2013sample}. In particular, there are concerns that algorithmic methods might overfit the landscape graph if learned using genetic data from an insufficient number of nodes. By varying the amount of data available in a sequence of large synthetic data experiments, we empirically explore the precise number of samples required to obtain a generalizing solution, showing that in some cases, as few as 25 populations are needed to reliably fit the parameters of a landscape graph involving 1000s of nodes.

\textbf{Additional related work.}
Relevant work on landscape genetics is included in Section \ref{sec:contributions}. We discuss additional related work on graph learning in Appendix C of this paper's full version \cite{dharangutte2020graph}.

\section{Proposed Method}
\label{sec:proposed_method}
We first describe notation needed to formalize the \emph{inverse landscape genetics} problem from Section \ref{sec:contributions}.

\textbf{Graph and Genetic Data Notation.} We denote the weighted, undirected landscape graph by $G = (V,E,w)$, where $V = \{v_1,\ldots, v_n\}$ is the vertex set, $E$ is the edge set, and $w$ is a vector of weights assigned to each edge. Let $m$ denote $m = |E|$. Typically $m \ll \binom{n}{2}$ since for most landscapes $G$ will be a grid graph with $m = O(n)$. We index both $E$ and $w$ by their terminal nodes: edges are $e_{i_1j_1}, \ldots, e_{i_mj_m}$ and weights are $w_{i_1j_1}, \ldots, w_{i_mj_m}$. It is often helpful to view graphs as electrical networks where $e_{ij}$ represents an electrical connection with \emph{conductance} $w_{ij}$ between nodes $v_i$ and $v_j$ \cite{spielman2011graph}. Let $r_{ij} = 1/w_{ij}$ denote the resistance of the connection. 

For a subset $S \subseteq V$ of nodes we have measured vectors of population genetic data $x_1,\ldots,x_{|S|}\in \R^d$. We only interact with this data through a black-box measure of genetic \emph{dissimilarity}: the specific choice is not important. In keeping with prior work, our experiments use the fixation index, typically denoted $F_{\text{ST}}$. For two populations, $i$ and $j$ a high $F_{\text{ST}}$ (close to 1) indicates greater difference between the measured genetic information in $x_i$ and $x_j$. Let $F \in \R^{|S|\times |S|}$ contain pairwise $F_{\text{ST}}$ (or another dissimilarity) for all nodes in $S$. Let $F_{i,i} = 0$ for all diagonal entries.


It has been established that the values in $F$ will correlate well with the \emph{effective resistances} of an appropriately chosen landscape graph $G$ \cite{mcrae:2006}. 
To define these measures, let $D\in \R_+^{n\times n}$ be the diagonal degree matrix with $D_{i,i} = \sum_{j: e_{ij} \in E} w_{ij}$. Let $A$ be the adjacency matrix with $A_{ji} = A_{ij} = w_{ij}$ for all $e_{ij}\in E$, and $0$ otherwise. Let $L$ be the weighted graph Laplacian as $D - A$.

\begin{definition}[Effective resistance] The effective resistance $R_{ij}$ between two nodes $i$ and $j$ satisfies
\begin{align*}
    R_{ij} = b_{ij}^T L^+ b_{ij}
\end{align*}
where $L^+$ is the Moore-Pensore pseudoinverse of the graph laplacian $L$ and $b_{ij} \in \R^n$ is the vector with $1$ at position $i$, $-1$ at position $j$ and $0$'s elsewhere. 
\end{definition}
The effective resistance between two nodes $v_i$ and $v_j$ is lower when there exist more low-resistance paths (i.e., high weight paths) between $v_i$ and $v_j$. It is known to be equal to the \emph{commute time} between $v_i$ and $v_j$ for a random walk with steps taken proportional to edges weights \cite{chandra1996electrical}, which gives some intuition for why the measure effectively quantifies organism dispersal through a landscape. We refer the reader to \citet{circuitscape_white} for further discussion of the important of effective resistances in landscape ecology. 

Let $R$ be the matrix of all pairwise effective resistances and note that $R_{ij} = R_{ji}$, which can be thought of as the resistance surface for the landscape. $R$ is $0$ along its diagonal. Let $R_S \in |S|\times |S|$ be the principal submatrix of $R$ containing only the pairwise effective resistances between nodes in $S$. The main problem we study is as follows:

\smallskip
\begin{mdframed}[backgroundcolor=black!10]
\vspace{.25em}
\begin{problem}[Inverse Landscape Genetics]\label{prob2:ilg}
Given landscape graph nodes $V$ and edges $E$ we are given a vector of environmental parameters $C_{i_kj_k} \in \R^q$ for each $e_{i_kj_k} \in E$ and a function class $\mathcal{P}$ from $\R^q \rightarrow \R_+$ which maps these parameters to a weight for each edge. Assume $\mathcal{P}$ is parameterized by parameters $\theta$ and denote functions in the class by $p_{\theta} \in \mathcal{P}$. For $p_\theta$, let $p_\theta(E)= [p_\theta(C_{i_1j_1}),\ldots, p_\theta(C_{i_mj_m})]$. Our goal is to find $\theta^*$ minimizing the loss:
\begin{align}
\label{optim}
    \theta^* = \argmin_{\theta} \mathcal{L}(\theta) =  \argmin_{\theta}\|R_S(p_\theta(E)) - F\|_F^2,
\end{align}
where $R_S(p_\theta(E))$ is the effective resistance matrix for the graph $G = (V,E, p_\theta(E))$ (restricted to nodes in $S$). $\left\lVert A \right\rVert_F^2 = \sum_{i} \sum_{j} A_{ij}^2$ denotes the standard Frobenius norm. 

\end{problem}
\vspace{.25em}
\end{mdframed}
Note that both $R_S(p_\theta(E))$ and $F$ have zeros on the diagonal, so the Frobenius norm above is equal to $2\times$ the standard squared loss between effective resistances and genetic dissimilarities. Other natural choices could be used instead of $\mathcal{L}$, e.g. the inverse of the Mantel correlation between $R_S(p_\theta(E))$ and $F$ \cite{graves2013}. In either case, the goal is to find edge weights such that the landscape graph $G$ induces effective resistances between nodes in $S$ which are as close as possible to the genetic dissimilarities in $F$. Alternatively, under the assumption that genetic dissimilarities represent noisy measurements of the \emph{true effective resistances for some unknown landscape graph $G^*$}, then Problem \ref{prob2:ilg} can be viewed as the task of recovering that graph.



\subsection{Example functional forms}
The problem is stated under the constraint that weights in the learned graph are a function $p_\theta$ of $q$ environmental parameters $C_{i_kj_k}$ about each edge. This function can take any form: we only require that it is differentiable with respect to its parameters. For example, prior work often considers $C_{i_kj_k}$ which is a single continuous scale parameter like edge elevation or temperature. A typical choice (see e.g. \cite{graves2013}) is to assume that $1/w_{i_kj_k} = r_{i_kj_k}$ follows an inverted Gaussian relation governed by parameters $\theta = [\beta$, $\beta_{opt}$ and $\beta_{SD}]$:
\begin{equation}\label{elev}
    \frac{1}{w_{i_kj_k}} = r_{i_kj_k} = \beta + 1 - \beta \exp\left({\frac{-(C_{i_kj_k} - \beta_{\text{opt}})^2}{2 \beta_{\text{SD}}^2}}\right)
\end{equation}
This form captures the fact that many species have for example a preferred ``ideal'' elevation $\beta_{\text{opt}}$ and are more likely to travel along edges of similar elevation: the resistance to dispersal $r_{i_kj_k}$ increases as $C_{i_kj_k}$ moves further from $\beta_{\text{opt}}$. Other papers consider slightly different functions, but they typically have the same general structure as \eqref{elev} \cite{peterman2018resistancega}. 



Another common functional form is linear. We simply let:
\begin{equation}\label{landcover}
    \frac{1}{w_{i_kj_k}} = r_{i_kj_k} = \alpha^T C_{i_kj_k}.
\end{equation}
For instance $C_{i_kj_k}$ might contain one-hot-encoded categorical data indicating what landcover an edge traverses (e.g. water, marshland, tundra). Each entry in $\alpha$ is a scalar associated with each category type that conveys how permeable the category is for movement.  
When continuous and discrete data at nodes is considered in unison, it is natural to add multiple functional forms linearly: e.g. we might have that $r_{i_kj_k} = r^\text{E}_{i_kj_k} + r^{\text{LC}}_{i_kj_k}$ where $r^\text{E}_{i_kj_k}$ is an elevation term in the form of \eqref{elev} and $r^{\text{LC}}_{i_kj_k}$ is a landcover term in the form of \eqref{landcover}. 

\subsection{Gradient computation}
Due to its non-convex nature, these is no closed form solution for \eqref{optim}.
The cornerstone of our approach is to instead find approximate solution by using projected gradient descent to minimize $\mathcal{L}(\theta)$. To do so, we need an efficient method for computing the gradient of this loss. 


 

\begin{prop}\label{prop}
Let $n_\theta$ denote the number of parameters in $\theta$ (typically a small constant) and let $J \in \R^{m\times n_\theta}$ denote the Jacobian with $J_{k,h} = \frac{\partial w_{i_kj_k}}{\partial \theta_h}$. Let $B\in \R^{m\times n}$ denote the edge-vertex incidence matrix of $G$ with $k^\text{th}$ row equal to $e_{i_k} - e_{j_k}$ where $i_k$ and $j_k$ are the terminal nodes of $G$'s $k^\text{th}$ edge.
\begin{align*}
    \nabla_\theta(\mathcal{L}) = \sum_{v_l,v_k \in S} \left( F_{lk}  - b_{lk}^T L_\theta^+b_{lk}\right)\cdot 2J^T \cdot(B L_\theta^+ b_{lk})^{\circ 2},
\end{align*}
where $^{\circ 2}$ denote the Hadamard power (i.e. square every vector element entrywise) and $L_\theta$ denotes the Laplacian of the landscape graph with edge weights $p_\theta(E)$

\end{prop}


\begin{proof} 
For given parameters $\theta$, let $w^{\theta} = p_\theta(E)$, where $p_\theta(E)$ is as defined in Problem \ref{prob2:ilg}. We have:
\begin{align}\label{grad}
    \nabla_{w^{\theta}}\mathcal{L} = - 2 \sum_{v_l,v_k \in S} \left( F_{lk}  - R(w^{\theta})_{lk}\right)\cdot \nabla_{w^{\theta}} R(w^{\theta})_{lk}
\end{align}
As in Problem \ref{prob2:ilg}, $R(w^{\theta})$ is the matrix of all pairwise effective resistances for the graph $G = (V,E, w^{\theta})$. From the definition for effective resistance, $R(w^{\theta})_{lk} = b_{lk}^T L_\theta^+ b_{lk}$. As in \cite{hoskins2018learning}, we can obtain a partial derivative for entries of $L^+$ with respect to $w^{\theta}$ via the Sherman-Morrison formula for rank one updates to the pseudoinverse. Specifically, we have $\frac{\partial L_\theta^+}{\partial w^{\theta}_{ij}} = -L_\theta^+ b_{ij} b_{ij}^T L_\theta^+$ and thus
\begin{align*}
    \frac{\partial R(w^{\theta})_{lk}}{\partial w^{\theta}_{ij}} = - b_{lk}^T \left(L_\theta^+ b_{ij} b_{ij}^T L_\theta^+\right) b_{lk} = - (b_{ij}^T L_\theta^+b_{lk})^2
\end{align*}
It follows that $\nabla_{w^{\theta}} R(w^{\theta})_{lk} = - (B L_\theta^+ b_{lk})^{\circ 2}$. The proposition follows from plugging this equation into \eqref{grad} and noting that $\nabla_\theta(\mathcal{L}) = J^T \cdot \nabla_{w^{\theta}}\mathcal{L}.$
\end{proof}



\textbf{Efficient computation of the gradient.} 
Proposition \ref{prop} yields an efficient algorithm for computing $\nabla_\theta \mathcal{L}$. In particular, since $n_\theta$ is typically a small constant computing the Jacobian $J$ is efficient for any differentiable functional form $p_\theta$. Then, ignoring the cost of computing $b_{lk}^TL_\theta^+ = L_\theta^+ b_{lk}$ for all $v_l,v_k \in S$, the gradient can be computed in $O(|S|^2\cdot m \cdot n_\theta)$ time. Note that since every row in $B$ is 2 sparse, $b_{lk}^T L^+ B$ can be computed in $O(m)$ time once $b_{lk}^T L^+$ is computed. 
Since $m = O(n)$ in most landscape genetics applications, the bottle neck is therefore computing each $L_\theta^+ b_{lk}$. 

This would naively require inverting the $n\times n$ Laplacian $L_\theta$, which would be computationally intensive and impractical for large graphs. We instead approximate the matrix-vector product $L_\theta^+ b_{lk}$ using an iterative solver for positive semidefinite linear systems ($L$ is positive semidefinite). In our experiments we use the standard MINRES method. To optimize the approach further, we note that $L_\theta^+ b_{lk} = L_\theta^+e_l - L_\theta^+e_k$ where $e_l$ and $e_k$ are the $l^\text{th}$ and $k^\text{th}$ standard basis vectors. Accordingly, we only need to solve $|S|$ linear systems (either $e_l$ as the right hand side for all $v_l\in S$), and can then recombine those solutions to return all $\binom{|S|}{2}$ vectors $L(w)^+ b_{lk}$ needed for the gradient computation. 

Each linear system solve could be further optimized by constructing e.g., a multigrid or partial Choleksy preconditioner. However, we found that the MINRES converged quickly in experiments without preconditioning, so the possible improvement is relatively small.


\begin{figure*}[htbp!]
\parbox{\textwidth}{
\begin{subfigure}{0.33\textwidth}
  \centering
    \includegraphics[scale=0.23]{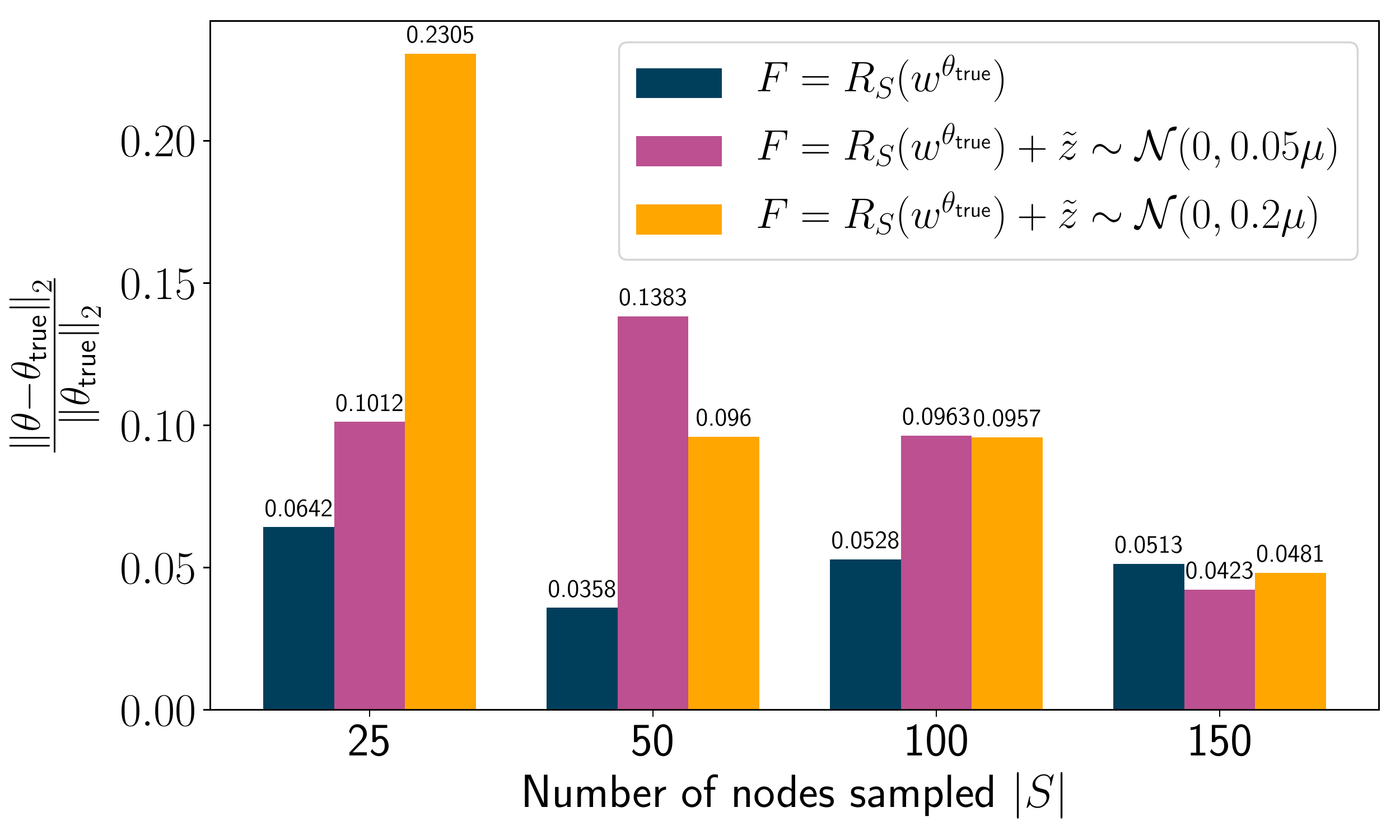}
    \caption{Discrete and combined case}
\end{subfigure}
\begin{subfigure}{.33\textwidth}
  \centering
  \includegraphics[scale=0.23]{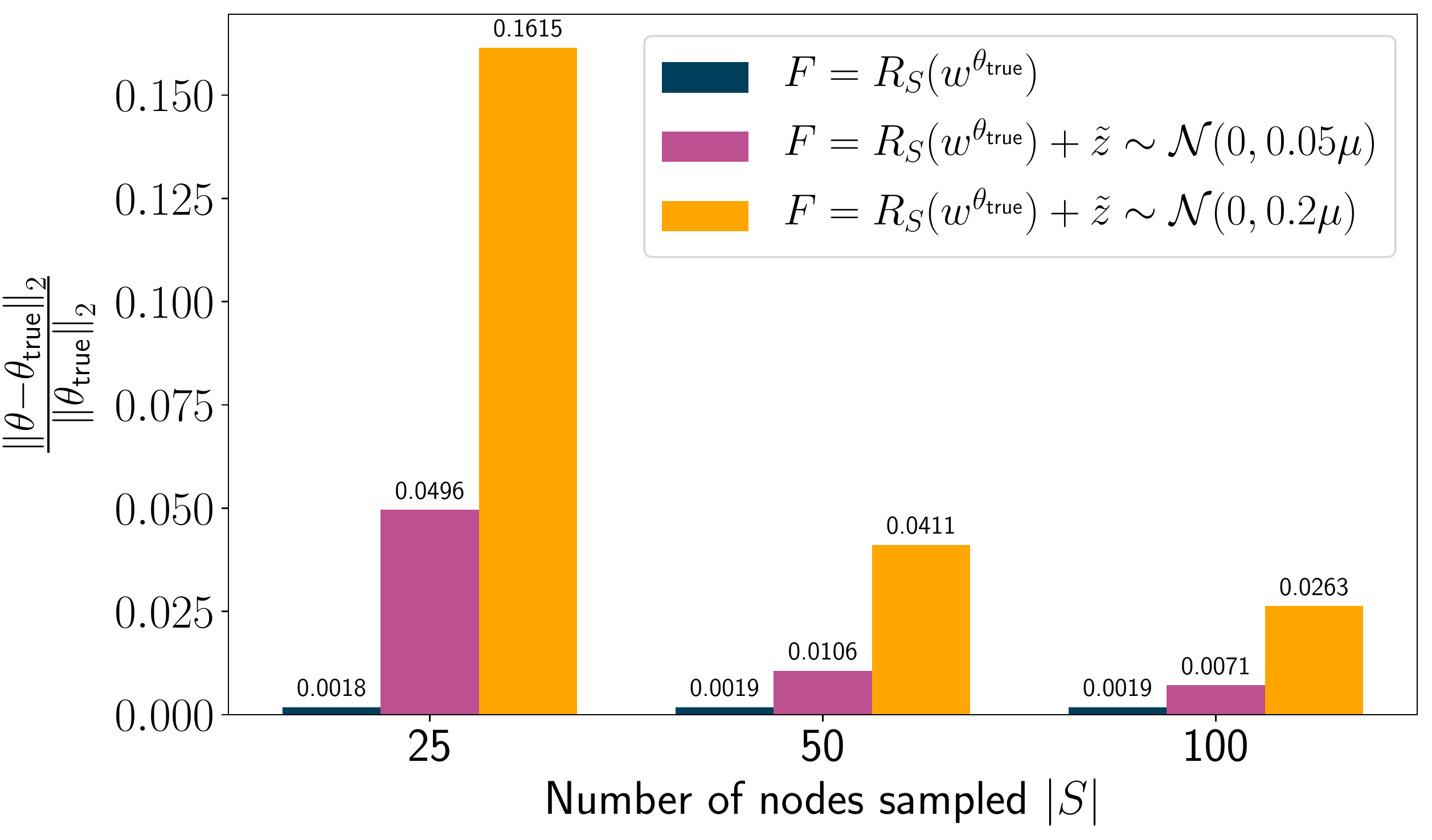}
  \caption{Discrete case (landcover data)}
  \label{fig:alpha}
\end{subfigure}%
\begin{subfigure}{.33\textwidth}
  \centering
  \includegraphics[scale=0.23]{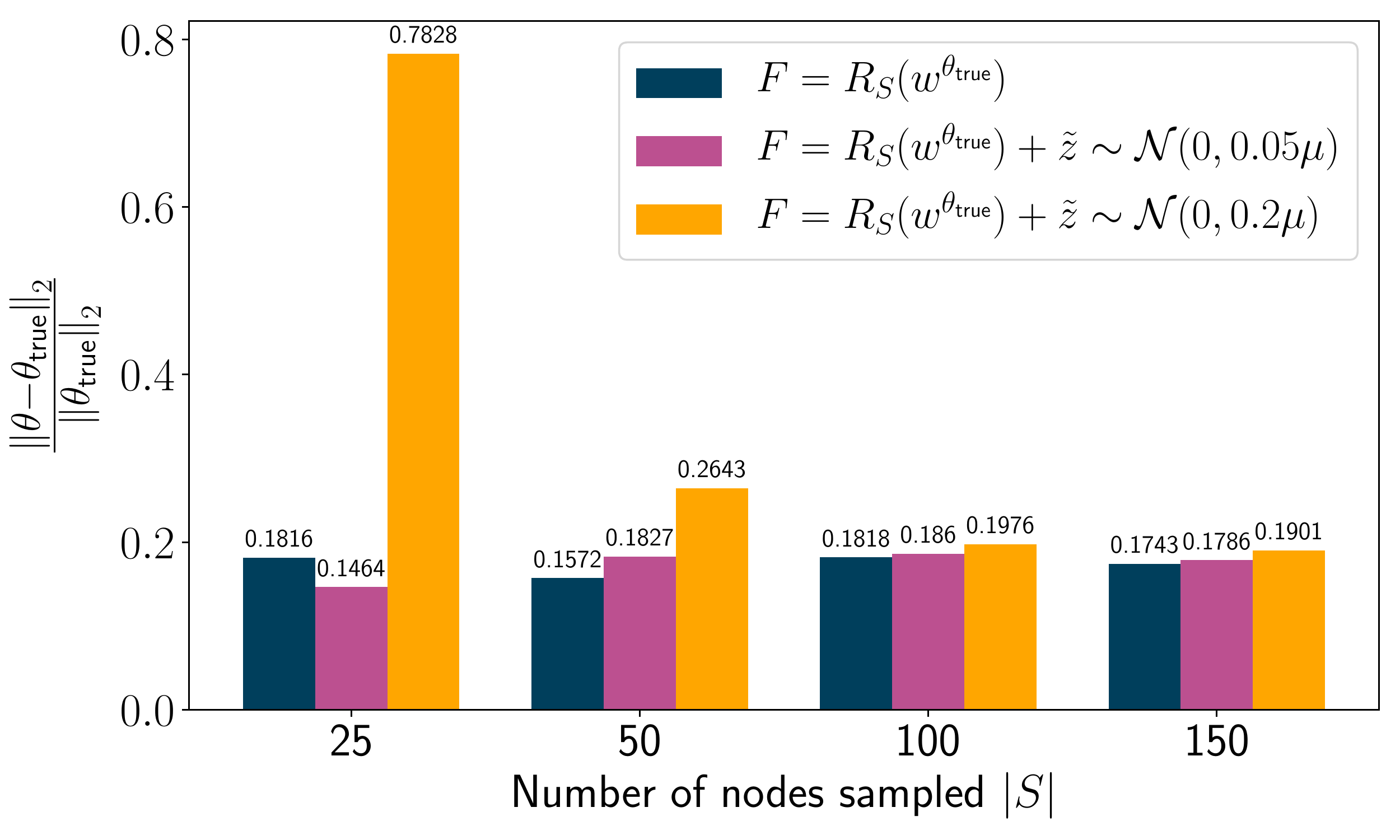}
  \caption{Continuous case (elevation data)}
  \label{fig:beta}
\end{subfigure}
}
\caption{Relative error between recovered parameters and true parameters for synthetic data experiments with different numbers of nodes sampled $N$ and noise standard deviation $\Tilde{\sigma}$. Parameter recovery improves with more samples (i.e., more locations with genetic similarity data), and generally with less noise (i.e., more highly correlated resistance and genetic data).}
\label{fig:combined}
\end{figure*}

\section{Empirical Results}
\label{sec:experiments}
With an efficient gradient oracle in hand for the loss function in Problem \ref{prob2:ilg}, we test a gradient based optimization approach on both synthetic and real genetic data. Real genetic data is obtained for the North American wolverine (Gulo gulo) from \citet{kyle:2001}, which provides $F_{\text{ST}}$ values for 6 populations living across a region in Alaska. Our goal is to understand the interplay between genetic variation in this region and the underlying landscape. Specifically, we obtain elevation data \cite{alaska_DEM} and land cover data \cite{homer2020conterminous}, which will be used as the basis for selecting edge weights in a landscape graph.



The landscape graph graph is constructed by dividing the Alaska region into a grid of square cells. In previous landscape genetics studies of the North American wolverine, cell sizes of 5 km and 50 km have been used \cite{mcrae:2007}. We choose a resolution of 15 km, which lead to a graph $G = (V,E)$ with $|V|=24035$ and $|E|=47746$. For each cell we create a node in the grid graph, and connect adjacent nodes with edges (as in Figure \ref{fig:imageToGridGraph}). Our landscape data comes as raster images, with each pixel corresponding to a region of $100 \times 100$ meters for elevation data and $30 \times 30$ meters for landcover data, so we have multiple pixels of information within each landscape cell. This data was resampled to cell resolution using standard GIS methods (see Appendix A in \cite{dharangutte2020graph} for details).

Continuous and discrete environmental parameters are then collected for each edge in the graph. For edge $k$, edge elevation
$C_{i_kj_k}^\text{E}$ is taken as the average elevation at cells $i$ and $j$ and scaled to lie within the range 0-10. For each edge we also construct a vector of one-hot-encoded landcover data $C_{i_kj_k}^{\text{LC}}$, which has 17 entries for landcover types like evergreen forest, barren land, or open water. Each entry in $C_{i_kj_k}^{\text{LC}}$ is given values as follows: 0 if the landcover type is absent at cell $i$ and $j$, 0.5 if present at either cell $i$ or $j$, or 1 if present at both cells $i$ and $j$. We model edge weights as a function of these parameters by linearly combining equation \eqref{elev} for elevation data and \eqref{landcover} for landcover data. So, the final parameter vector we hope to learn when solving Problem \ref{prob2:ilg} is $\theta = \{\beta, \beta_{\text{opt}}, \beta_{\text{SD}}, \alpha\in\R^{17}\}$.

To minimize \eqref{optim}, we implement a projected gradient descent method with RMSProp step size adjustment, which adjusts learning rate by a decaying average of squared gradients \cite{Tieleman2012}. 
Since edge weights are constrained to be non-negative, and all edge data is non-negative, we project parameters to $\max(\epsilon,\theta)$ with $0 < \epsilon \leq 1$ at each gradient step. This ensures non-zero resistance value for all landcover types, which is a constraint often imposed in prior work.
All experiments were run on server with 2vCPU @2.2GHz and 13 GB main memory. 


\textbf{Synthetic data}: Our first set of data experiments uses the \emph{real landscape data} from Alaska, but in conjunction with carefully \emph{simulated genetic data}, which makes it possible to better assess the performance of our method. Specifically, we selected a random $50\times 50$ subgrid of our Alaska graph to obtain a grid graph with $|V|=2500$ and $|E|=4900$. We then constructed a \emph{ground truth} graph by randomly sampling a set of parameters, $\theta_{\text{true}}$, and evaluating the weights for all edges in $E$. The goal in our synthetic experiments is to recover this ground truth, which is a common set up in testing algorithms for inverse landscape genetics as real ground truth data is never available \cite{graves2013}.

In particular, we construct the pairwise effective resistance matrix 
$R_S(w^{\theta_{\text{true}}})$ for a set of nodes $S$ with $N = |S| \ll |V|$. For the nodes in $S$, we produce a simulated genetic similarity matrix $F$ by setting $F_{lk} = \left[R_S(w^{\theta_{\text{true}}})\right]_{lk} + \Tilde{z}$ where $\Tilde{z} \sim \mathcal{N}(0,\Tilde{\sigma})$. We run experiments with $\Tilde{\sigma} = \{0, 0.05\mu,0.2\mu\}$, where $\mu$ is the mean of the resistances in $R_S(w^{\theta_{\text{true}}})$. These cases (no, low, and high noise) range from perfect to poor alignment between genetic data and landscape resistance. For parameters $\theta$ obtained after optimization, we report the relative parameter error as $\left\lVert \theta-\theta_{\text{true}} \right\rVert_2/\left\lVert \theta_{\text{true}} \right\rVert_2$. We ignore parameters for landcover types present at less than $1\%$ of nodes as these parameters can't be determined with any level of accuracy (since they have essentially no impact on graph effective resistances). Results are shown in Fig. \ref{fig:combined}, with additional experiments in \cite{dharangutte2020graph}.

\begin{figure*}[h]
\parbox{\textwidth}{
\begin{subfigure}[]{0.5\textwidth}
  \centering
  \includegraphics[scale=0.24]{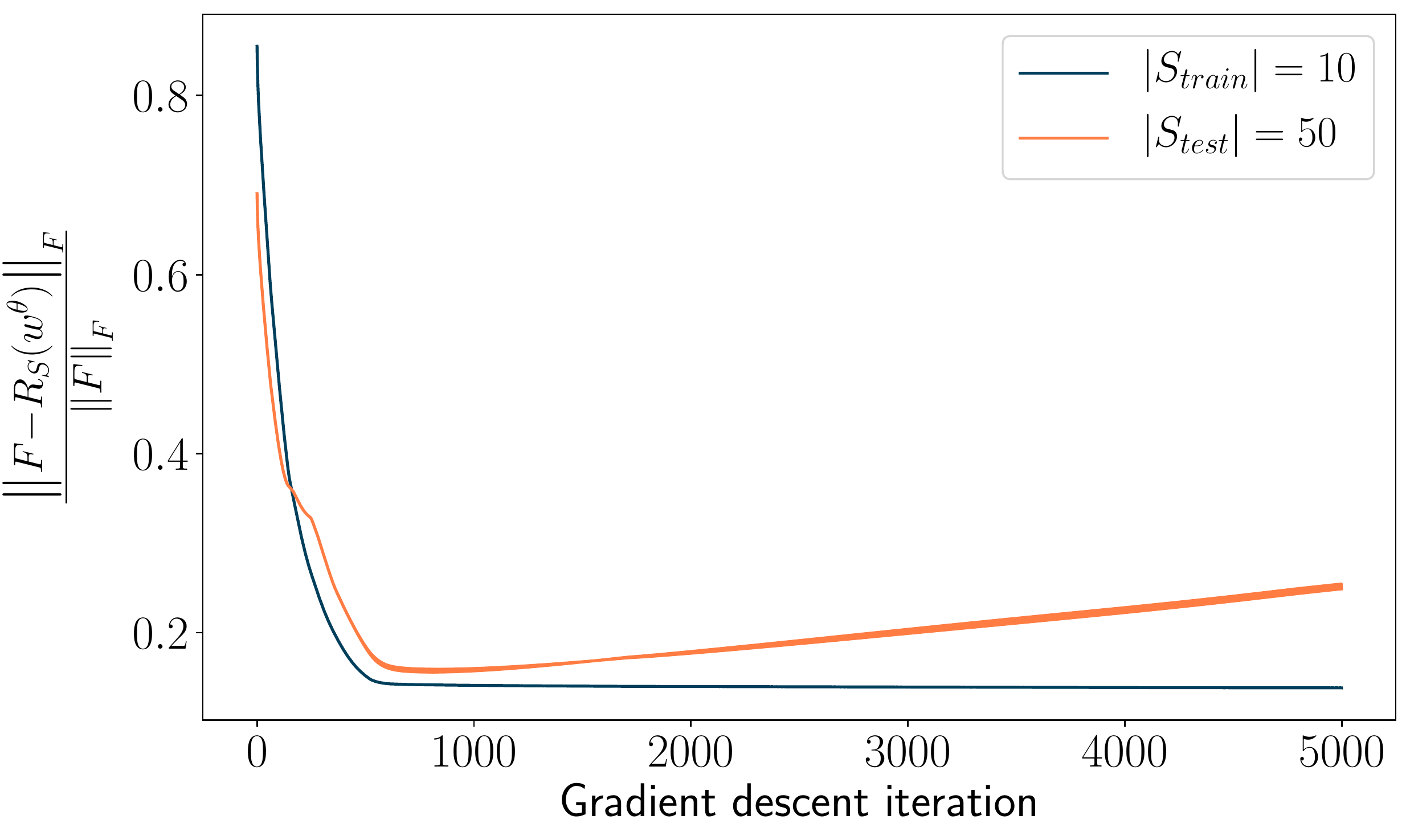}
  \caption{Relative loss. vs iteration for $N = 10$}
  \label{fig:train_test_10}
\end{subfigure}%
\begin{subfigure}[]{0.5\textwidth}
  \centering
  \includegraphics[scale=0.24]{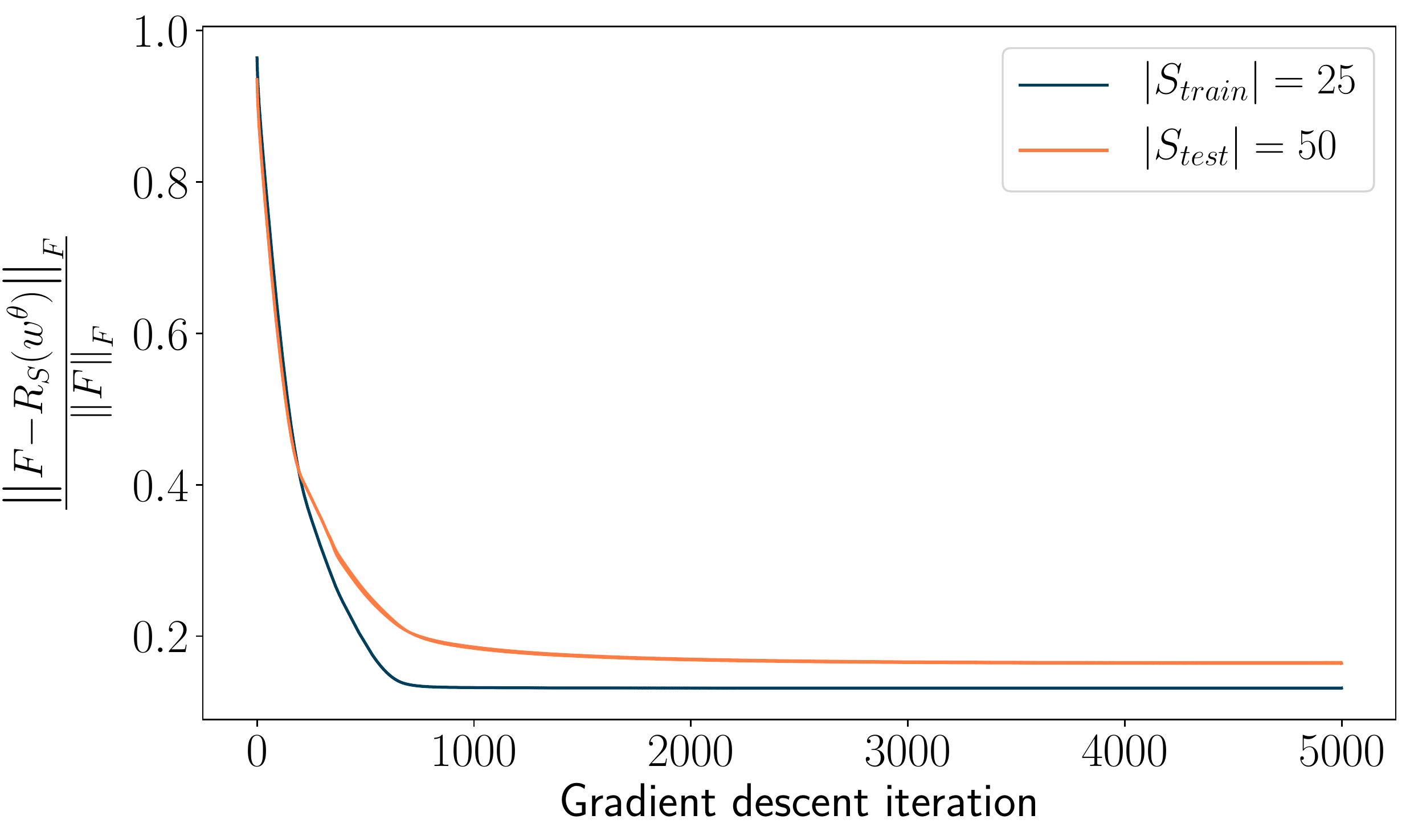}
  \caption{Relative loss vs. iteration for $N = 25$}
  \label{fig:train_test_25}
\end{subfigure}%
}
\caption{Train and test loss for different values of $N$ on synthetic data with $\Tilde{\sigma}=0.2\mu$. Parameters are learnt for nodes belonging to $S_{\text{train}}$ and used to infer pairwise effective resistance for nodes in $S_{\text{test}}$. We obtain good generalization for $N$ as low as 25, but observe clear overfitting for $N=10$.}
\label{fig:train_test}
\end{figure*}

We conclude that, as $N$ increases, our method obtains high quality approximations to the true parameters $\theta_{\text{true}}$, even in the high noise regime. For example, $N = 150$ was sufficient for fitting the graph parameters in all cases. This is a pretty typical number of samples for a landscape genetics study (e.g. \citet{shirk2010} obtain genetic data for mountain goats from $N = 149$ locations over a comparably sized area). Accordingly, even for a reasonably large number of landscape parameters, reliable learning of landscape data should be possible with existing data collection methods.

\begin{figure*}[h]
\parbox{\textwidth}{
\begin{subfigure}[]{0.5\textwidth}
  \centering
  \includegraphics[scale=0.25]{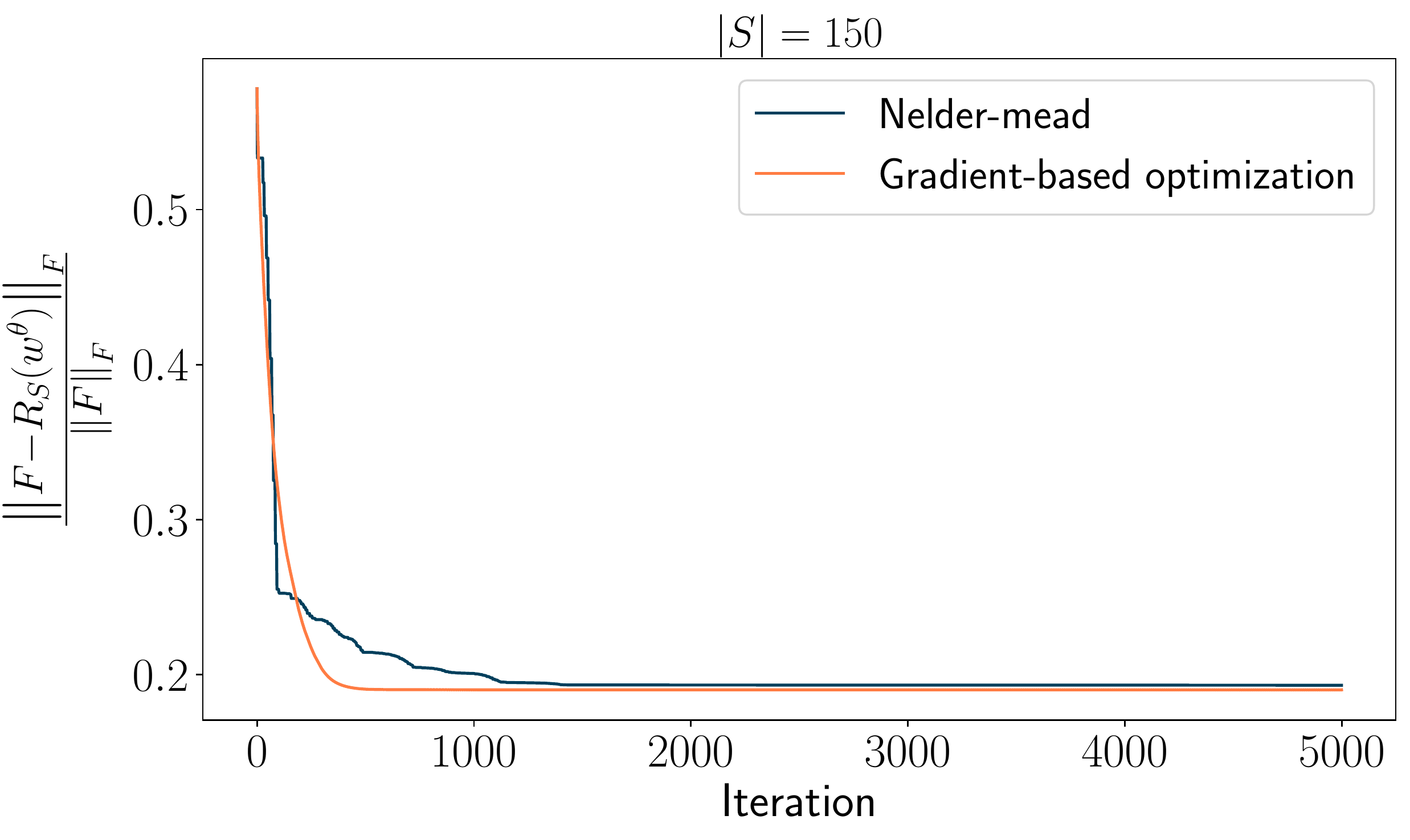}
  \caption{Relative loss vs. iteration for Nelder-Mead and gradient-based optimization. After 5000 iterations, the loss value is 0.193 for Nelder-Mead and 0.19 for gradient-based optimization.}
  \label{fig:nelder_mead_loss}
\end{subfigure}%
\begin{subfigure}[]{0.5\textwidth}
  \centering
  \includegraphics[scale=0.25]{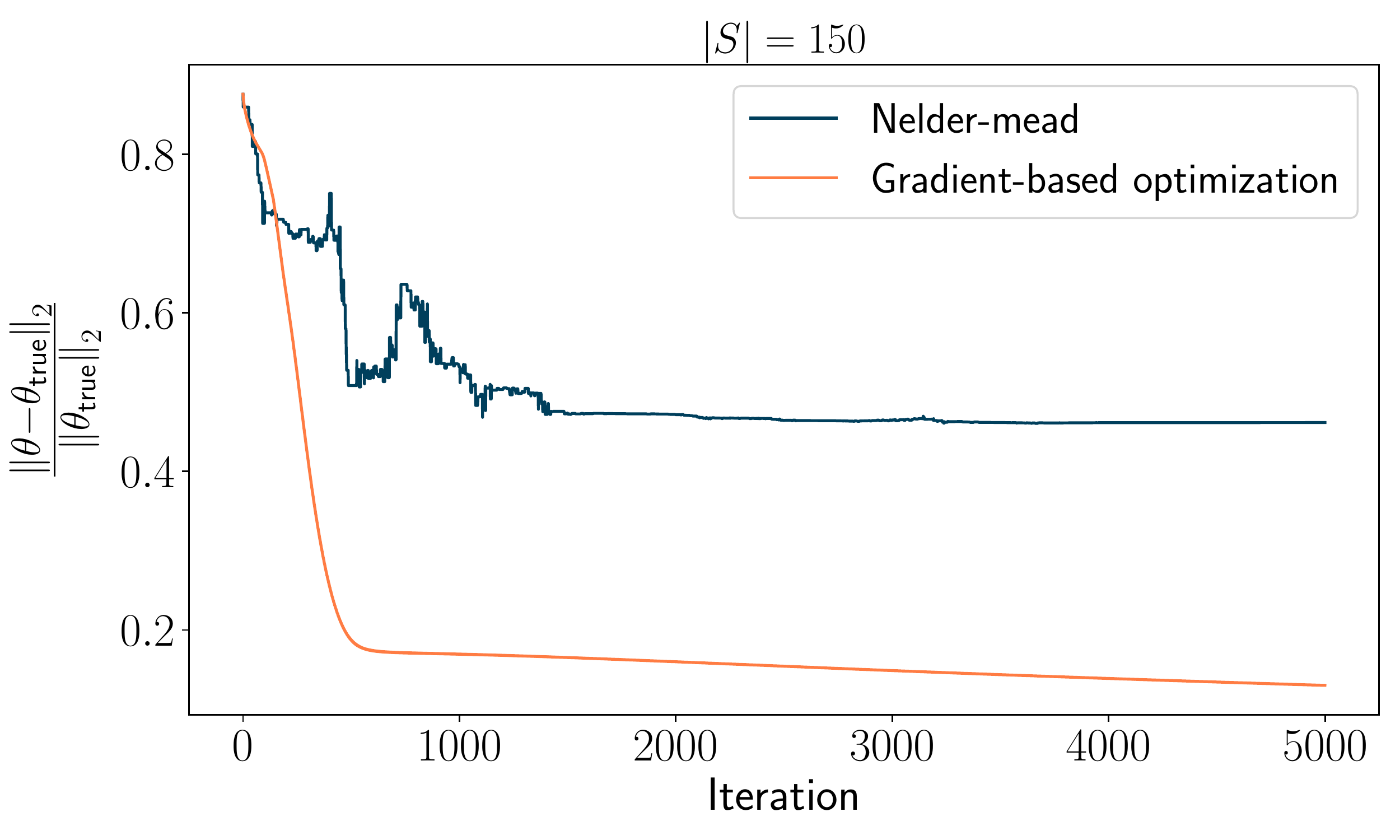}
  \caption{Relative parameter error between recovered parameters and true parameters with iteration. Gradient-based optimization is better at recovering true parameters.}
  \label{fig:nelder_mead_params}
\end{subfigure}}
\caption{Comparison of proposed method to a heuristic optimization technique. Gradient-based optimization is faster in convergence and better at recovering true parameters with enough data. Experiments are for synthetic data with high noise setting with $N=150$ and $\Tilde{\sigma} = 0.2\mu$, where $\mu$ is mean of entries in true resistance surface $R_S(w^{\theta_{\text{true}}})$ corresponding to nodes in $S$.}
\label{fig:nelder_mead}
\end{figure*}

\begin{figure*}[h]
\parbox{\textwidth}{
\begin{subfigure}{0.33\textwidth}
   \centering
    \includegraphics[scale=0.23]{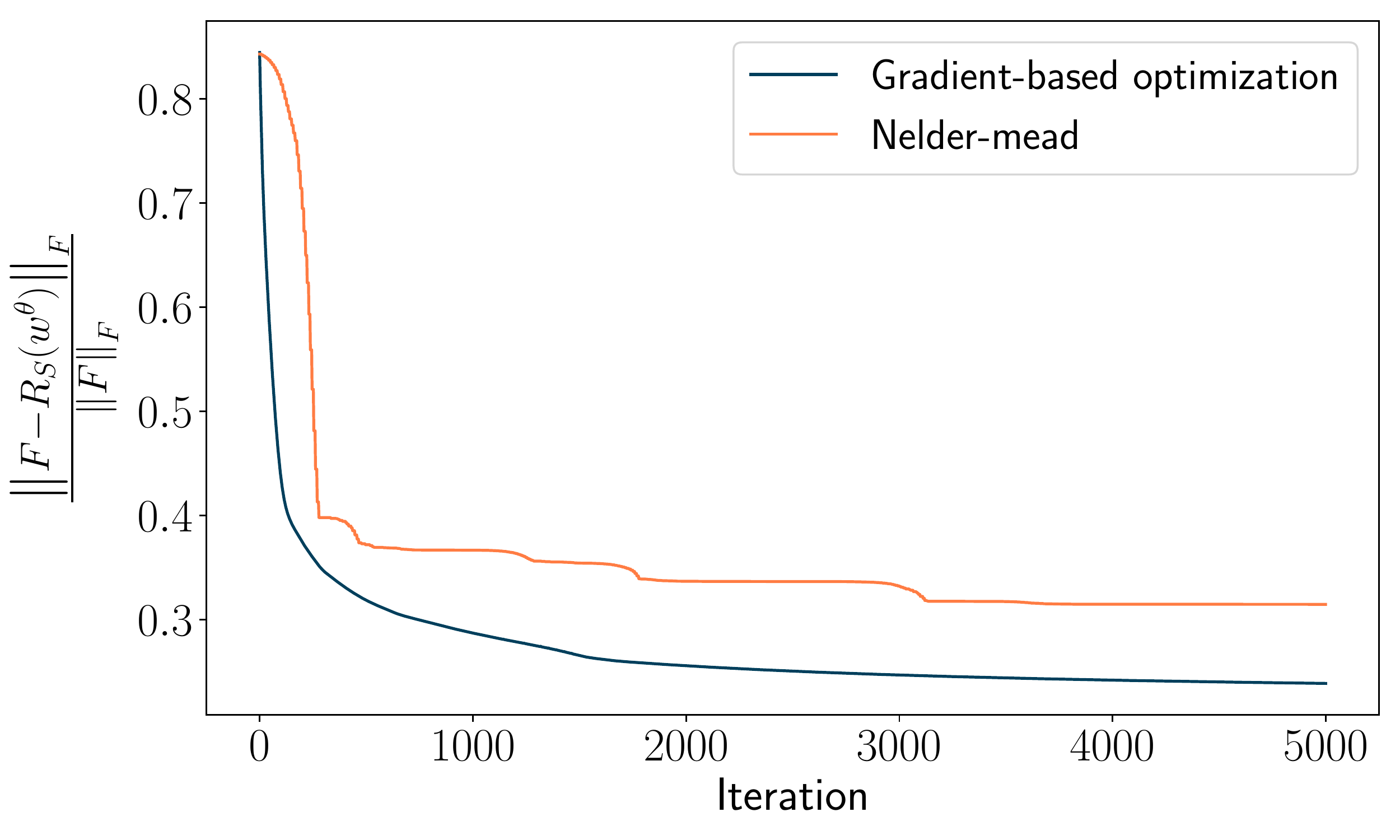}
    \caption{Relative objective function value vs iteration. Gradient-based optimization obtains a better solution faster.}
    \label{fig:wolverine_loss}
\end{subfigure}
\begin{subfigure}{0.33\textwidth}
  \centering
  \includegraphics[scale=0.23]{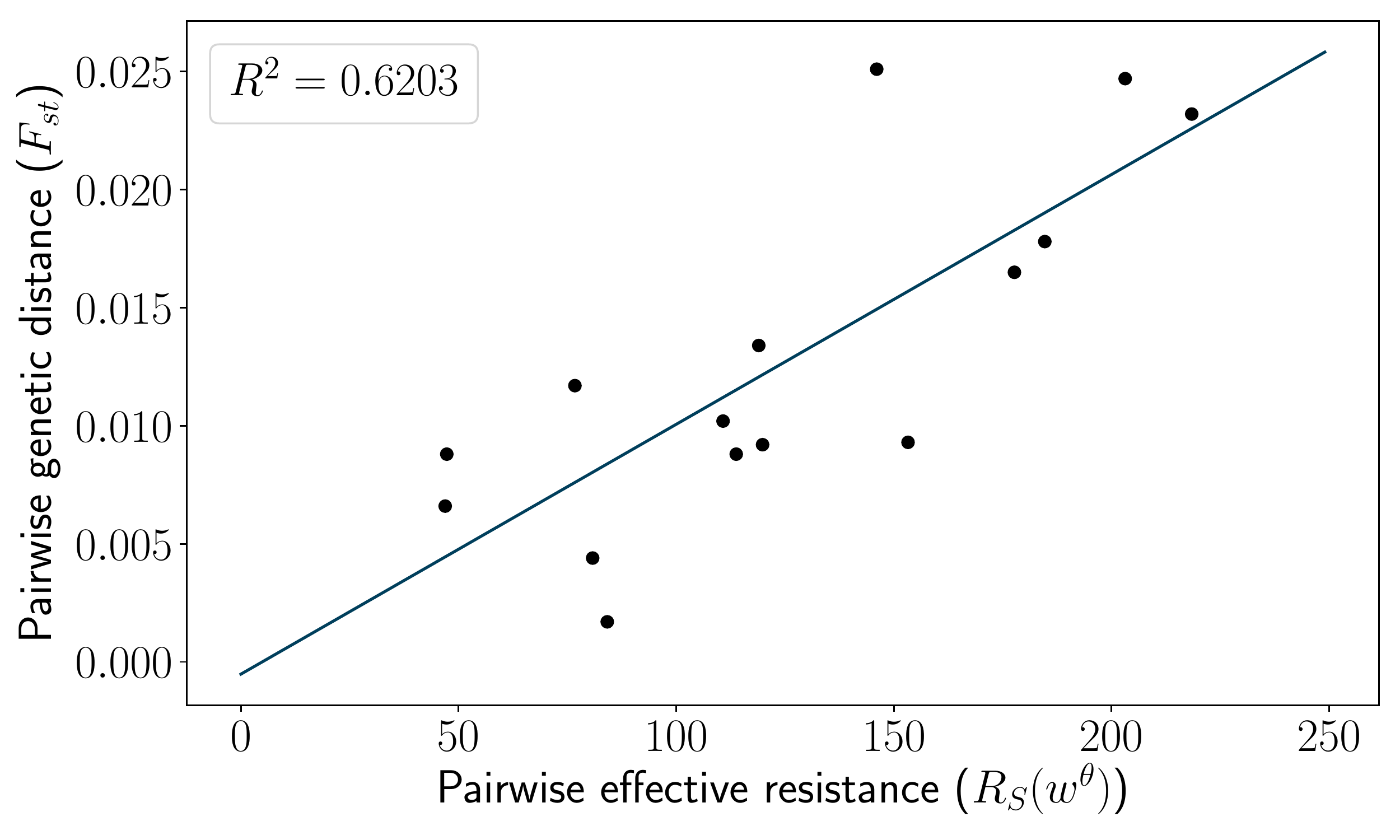}
  \caption{$F_{\text{ST}}$ vs effective resistance from learnt parameters using Nelder-Mead algorithm.}
  \label{fig:wolverine_nelder_mead}
\end{subfigure}%
\begin{subfigure}{0.33\textwidth}
  \centering
  \includegraphics[scale=0.23]{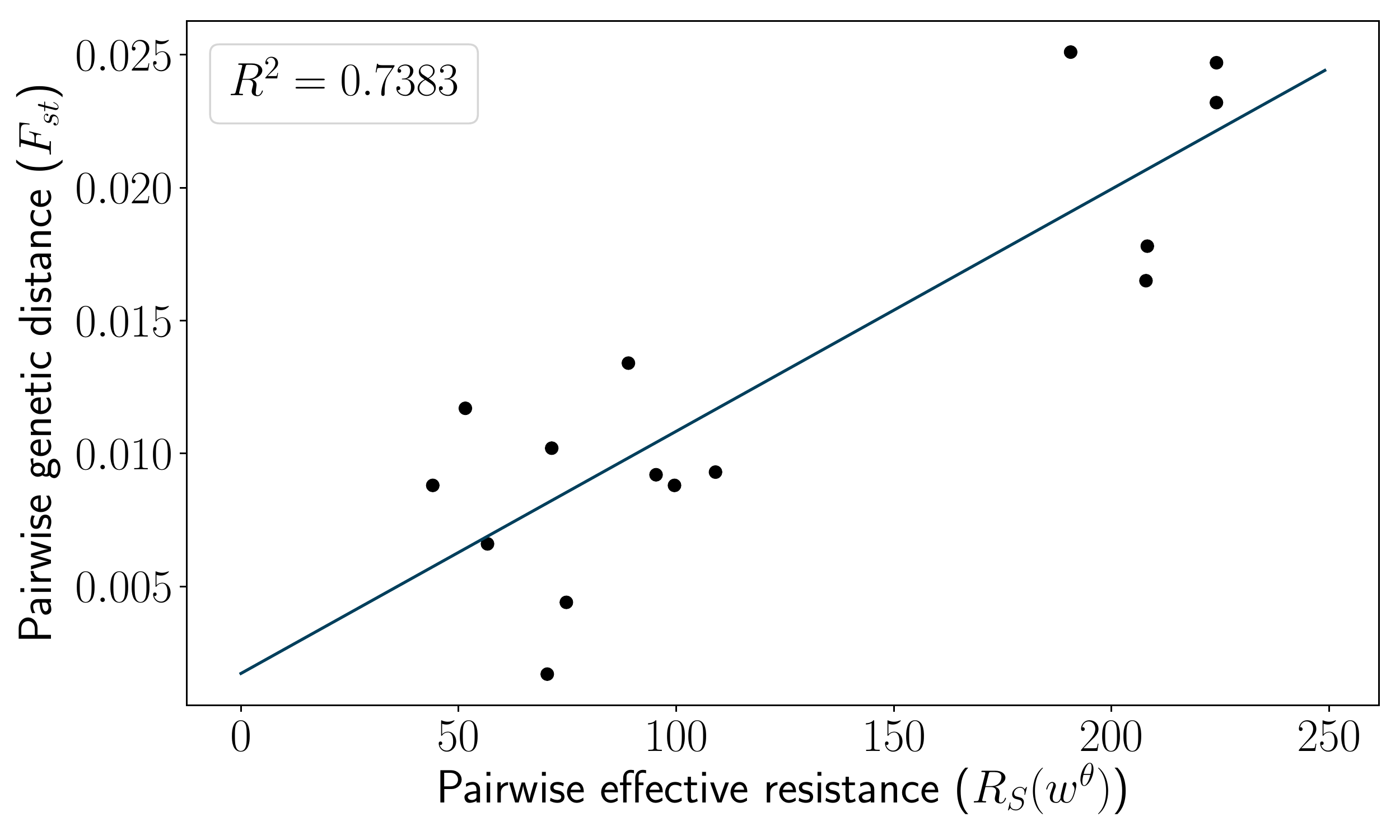}
  \caption{$F_{\text{ST}}$ vs effective resistance from learnt parameters using gradient-based optimization.}
  \label{fig:wolverine_grad_opt}
\end{subfigure}
}
\caption{Relative objective function value and $R^2$ values computed for a linear fit between $F_{ST}$ and effective resistances in the final learnt landscape graphs for real-world data. Gradient-based optimization obtains a slightly better fit.}
\label{fig:wolverine}
\end{figure*}

\textbf{Addressing overfitting:} It has been reported that a potential concern with optimizing landscape graphs is overfitting when $N$ is small. I.e., the landscape graph fit to $F$ does not generalize to new data \cite{oyler2013sample}. To validate against overfitting, we randomly split nodes into sets $S_{\text{train}}$ and $S_{\text{test}}$. We learn parameters $\theta$ for nodes in $S_{\text{train}}$ and evaluate these parameters against pairwise effective resistances in $S_{\text{test}}$. Even in the high noise setting, with $\Tilde{\sigma} = 0.2\mu$, test loss converges along with train loss when $N$ is as low as 25, (Fig. \ref{fig:train_test}). This implies good generalization and a lack of overfitting, \emph{even though we do not accurately recover all parameters in $\theta_{\text{true}}$}. This is not necessarily surprisingly: it indicates that, while the inverse landscape genetics problem may be poorly conditioned with respect to $\theta$ (as observed in \citet{graves2013}) it is still possible to obtain reliable predictive models with  little data.

\textbf{Comparison with existing approaches} : We compare gradient-based optimization to the Nelder-Mead method \footnote{Note that Nelder-Mead is an unconstrained optimization method, so we add a projection step to ensure interpretable parameters are returned. This does not noticeably affect the behavior of convergence in our experiments.}, which has been used in prior work on inverse landscape genetics \cite{graves2013}. We observe that our method is faster in terms of convergence and also better at recovering true parameters with enough data. Nelder-Mead eventually achieves comparable performance in terms of train loss but fails at recovering the true parameters (Figure \ref{fig:nelder_mead}). To ensure a fair comparison, we choose the same random initialization of parameters and non-negativity constraints.


\textbf{North American wolverine (gulo gulo)} : For experiments on real data,  $F_{\text{ST}}$ values range from 0 to 1 and we have access to genetic data at $15$ nodes out of 24035 nodes. After fitting $\theta$ with our gradient based method, we compute the $R^2$ value for a linear fit between recovered resistances and $F_{\text{ST}}$ values (Fig. \ref{fig:wolverine}), a metric used in prior work \cite{mcrae:2007}.
We obtain an $R^2$ value of 0.7383 using gradient-based optimization and 0.6203 using Nelder-mead, in comparison to 0.68 (5km resolution) and 0.71 (50km resolution) obtained by \citet{mcrae:2007} using expert opinions. Note that \citet{mcrae:2007} use a binary map as habitat/nonhabitat for underlying landscape with 12 populations whereas we use a multivariate surface with continuous and discrete data with 6 populations. We provide the final parameters $\theta$ in Table \ref{tab:params}. The solutions for Nelder-Mead optimization and our gradient method largely agree: landcover types that allow for movement under cover (e.g., forests) are assigned low resistances values, and open water is assigned the highest resistance. There is a notable difference between learned parameters for barren land, and sedge/herbaceous landscape, which would be interesting to explore further.


\begin{table}[H]
\centering
\begin{tabular}{ccc} \toprule
    Parameter & \specialcell{Nelder\\Mead} & \specialcell{Gradient\\based\\optimization} \\ \midrule
    $\beta$ & 0 & 0  \\
    $\beta_{opt}$ & 10 & 9  \\
    $\beta_{SD}$ & 0 &  0 \\
    Open water & 227 & 502 \\
    Barren Land & 151 & 5 \\
    Deciduous forest & 0 & 0\\ \bottomrule
\end{tabular}
\quad
\begin{tabular}{ccc} \toprule
    Parameter & \specialcell{Nelder\\Mead} & \specialcell{Gradient\\based\\optimization} \\ \midrule
    Evergreen forest & 0  & 12 \\
    Mixed forest & 0 &  0 \\
    Dwarf Shrub & 18 &  0 \\
    Shrub/Scrub & 107 & 95\\
    Sedge/Herbaceous & 0 & 500\\
    Woody Wetlands & 25 & 26\\ \bottomrule
\end{tabular}
\caption{
Final parameter values after optimization, rounded to nearest whole number. We do not report for landcover types which were present at less that $2\%$ of nodes in the graph. The $\beta$ parameters are for elevation data -- see equation \eqref{elev}.}
    \label{tab:params}
\end{table}


\section*{Conclusion and Future Work}
By formalizing the Inverse Landscape Genetics problem as a graph inference problem involving noisy measurements of effective resistances, we show how to apply powerful optimization methods from machine learning to this scientifically important problem. These methods already provide a promising alternative to existing heuristics, and will allow researchers to more efficiently and effectively solve real-world problems, or to explore synthetic problems at scale. This could facilitate, for example, more widespread investigations of the statistical complexity of inverse landscape genetics.

A major open research direction is to develop further theory around the problem formalized in this paper. For example, as discussed in \cite{hoskins2018learning}, while non-convex gradient descent methods seem to perform well, it remains unclear if Problem \ref{prob2:ilg} can be provably solved in polynomial time. 

In terms of statistical complexity, our problem is related to that of inferring graphical models \cite{attias2000variational,mohan2012structured}, which has been studied in different formulations across machine learning, statistics, and graph signal processing \cite{egilmez2017graph,ortega2018graph}. The common assumption is that the correlation matrix between data at graph nodes is related to the adjacency or Laplacian matrix of an unknown graph. Several works explore how many samples are needed to learn the structure of this graph, often under additional assumptions like graph sparsity \cite{Raskutti:2009,cai2011constrained}. 

Our work makes a structural assumption that the graph underlying our data has both a simple edge structure (i.e., its a grid graph) and that edges weights are functions of relatively low-dimensional edge data (i.e., landscape information). An interesting direction for future work is understanding if these natural assumptions can be used to formally bound the sample complexity of the inverse landscape genetics problem. 

Doing so will likely require a better understanding of \emph{how} samples should be collected for optimal inference. By choosing to collect organism samples in specific geographical locations, we often have control over exactly which graph nodes data is collected for. Empirically, sample design can have substantial impact on how much data is needed to solve the inverse landscape genetics problem \cite{oyler2013sample}. Again, we hope that our work provides a starting point for further exploration of this important question. Progress would allow researchers to more efficiently study the dispersion of at-risk species, for which it is difficult to collect substantial genetic data.
 
\section*{Acknowledgments}
The authors would like to thank Cameron Musco and Charalampos E. Tsourakakis for early discussions about this work, as well as Uthsav Chitra who provided assistance in formalizing the inverse landscape genetics problem studied. Funding in direct support of this work came solely from NYU's Tandon School of Engineering. There are no other relevant financial activities to disclose.

\section*{Ethics Statement}
As discussed in the introduction, we believe our work has high potential for positive broader impacts related to environmental protection and conversation. We also hope this paper highlights an interesting applications of graph-inference that we believe is not well known to the machine learning community. \citet{storfer2007putting} emphasize the need for forming bridges between research areas with different technical expertise to move landscape genetics forward. Already there has been successful cross-field collaboration between ecologists and those working in spatial statistics. We hope to bring the machine learning community into the fold. 

With those benefits in mind, our work does have some potential for negative impact. In particular, the graph inference problem studied has potential applications to de-anonymizing edges in social networks \cite{hoskins2018learning,liben2007link}. Our contributions would probably have limited impact on this sort of application (since our methods are developed specifically for parameterized, planar graphs used in modeling landscapes). Nevertheless, continued work in the area could have negative privacy implications.



\bibliography{references} 

\newpage
\appendix

\section{Experimental Details}
\label{sec:additional_details}
The resolution for elevation data is $100\times100$ meters and that of landcover data is $30\times30$ meters. To have a reasonably sized graph, we change the resolution and convert pixels to cell corresponding to 15km resolution, compared to previously used 5km and 50km \cite{mcrae:2007}, using the re-sampling tool from ArcMap with the `nearest' technique for land cover data and `bilinear' technique for elevation data\footnote{These are the recommended techniques from the software documentation. We refer the reader to \url{https://desktop.arcgis.com/en/arcmap/10.3/tools/data-management-toolbox/resample.htm}}. For experiments with real-world genetic data, we set the edge weight to 0 if any node for the edge had landcover type as `unclassified' whereas for synthetic data we set corresponding $\alpha$ parameter to high value which results in minuscule edge weight. For optimization, we initialize $\beta$ as 1 and all other parameters ($\theta_{\text{true}}$) randomly from a discrete uniform distribution, as a small initial value for $\beta$ provided us with most consistent results. For setting the true parameters $\theta_{\text{true}}$ and initialization, we sample $\beta_{\text{opt}}$ and $\beta_{\text{SD}}$ from range 0-10 and $\alpha$ from range 0-100 for synthetic data experiments. For experiments with real-world data, we initialize $\beta_{\text{opt}}$ and $\beta_{\text{SD}}$ similarly but for alpha we use the range 1-10. We run all experiments with RMSProp for 5000 iterations using learning rate of 0.1 and 0.9 for $\gamma$ parameter. For projecting parameters as $max(\epsilon, \theta)$, we use $\epsilon=1$ for synthetic data experiments and $\epsilon=10^{-20}$ for experiments with genetic data for North American wolverine for the parameters $\theta = \{\beta, \alpha\ \in \mathbb{R}^{17}\}$. For $\beta_{\text{SD}}$ we use $\epsilon = 10^{-3}$. As discussed in section 4, projecting to a small $\epsilon > 0$ instead of exactly 0 helps avoid numerical issues when computing the gradient and prevented the algorithm from getting stuck in local minima. 

Note that for reporting $R^2$ values for linear fit between pairwise effective resistances of learnt landscape graph and $F_{\text{ST}}$ values, \cite{mcrae:2007} use $\frac{F_{\text{ST}}}{1-F_{\text{ST}}}$ instead of $F_{\text{ST}}$ values whereas we use $F_{\text{ST}}$, but as $F_{\text{ST}} << 1$ for the data under consideration, the $R^2$ values are approximately equal.

\section{Additional Experiments}
\label{sec:additional_results}

\subsection{Synthetic data experiments}


We present results for experiments where we consider landscape data to be comprised only of elevation data (continuous surface), only of landcover data (discrete surface) and linear combination of both. We note that when only elevation data is considered, we consistently obtain good recovery for $\beta_{\text{opt}}$ and $\beta_{\text{SD}}$ but not for $\beta$. Although the recovery is not consistent across all parameters, $\beta_{\text{opt}}$ and $\beta_{\text{SD}}$ are typically more meaningful to researchers, indicating the preferred elevation range and range of elevation of the species. We conclude that as the number of nodes sampled($N$) increases, we obtain good approximations to the true parameters $\theta_{\text{true}}$.

For parameter recovery with different settings of $N$, we observe that $N=150$ is sufficient for reliable recovery of parameters. Appendix Figures \ref{fig:Combined}, \ref{fig:Betas} and \ref{fig:Alphas} show the relative parameter approximation error with gradient descent iteration. For landscape graph with only discrete data with $N=25$ and $\Tilde{\sigma}=0.2\mu$, we observe overfitting where the relative parameter error increases with iteration in Appendix Figure \ref{fig:alphas25}.


\begin{figure*}[hbtp!]
        \centering
        \begin{subfigure}[b]{0.49\textwidth}
            \centering
            \includegraphics[scale=0.23]{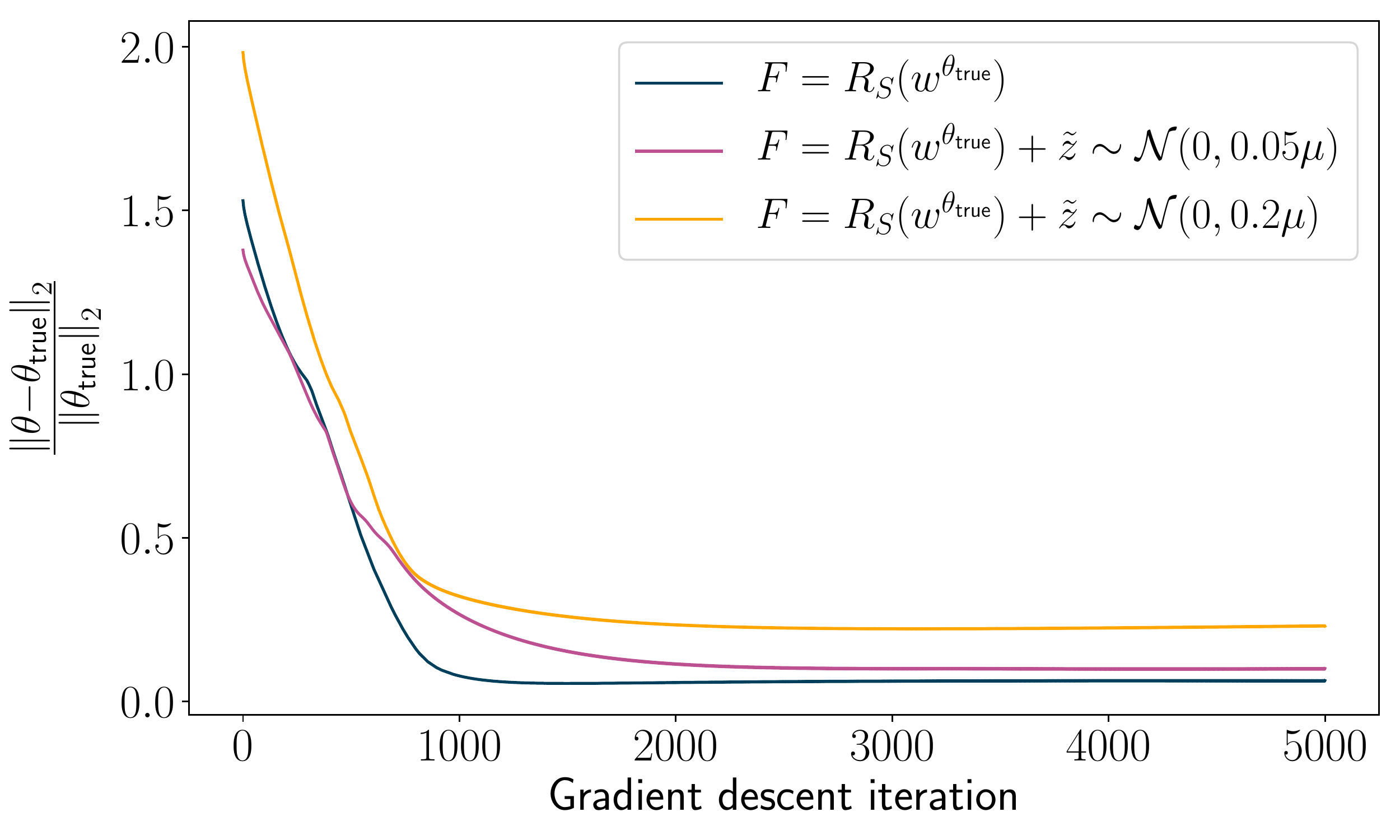}
            \caption[]%
            {{\small Landscape with continuous and discrete data with $N=25$.}}    
            \label{fig:combined25}
        \end{subfigure}
        \begin{subfigure}[b]{0.49\textwidth}  
            \centering 
            \includegraphics[scale=0.23]{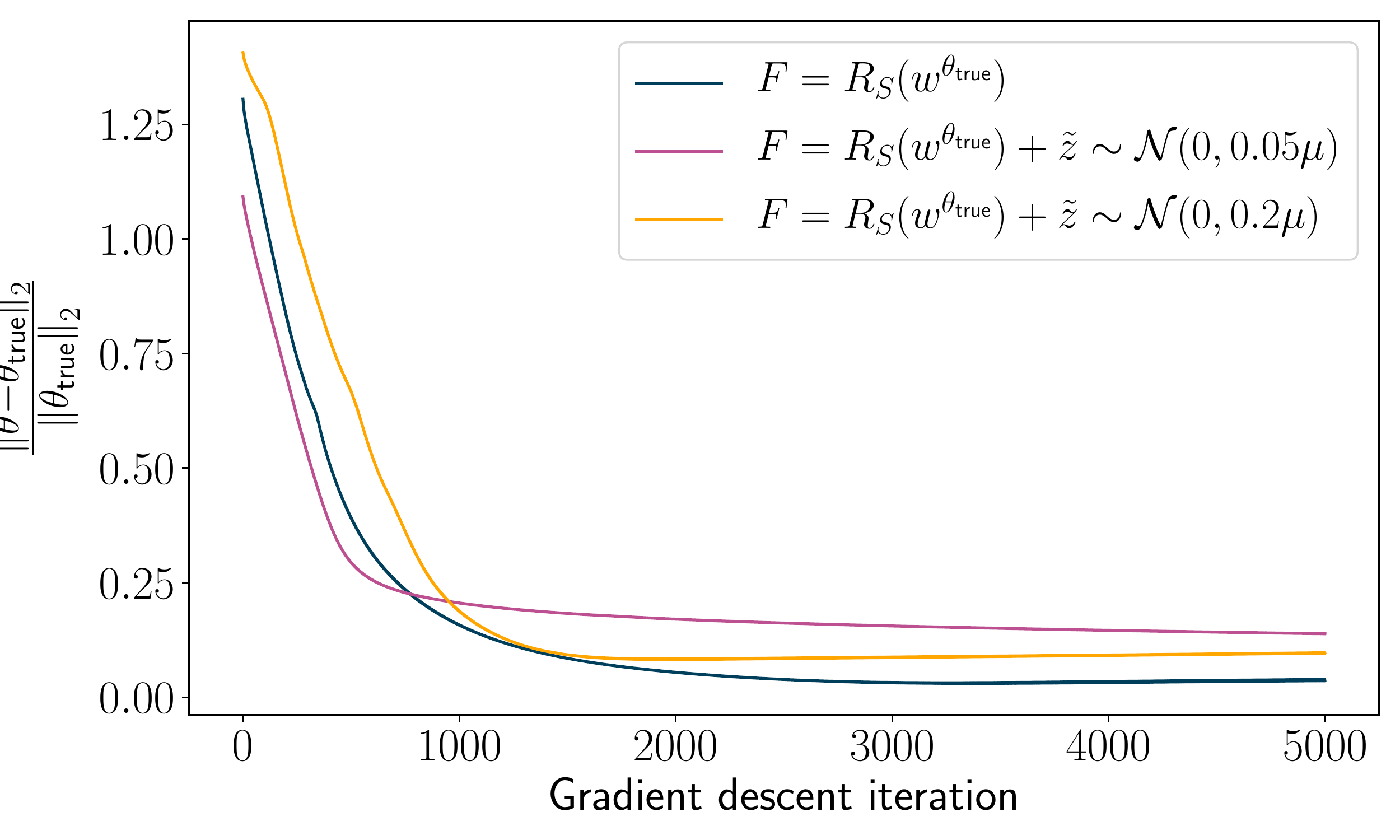}
            \caption[]%
            {{\small Landscape with continuous and discrete data with $N=50$.}}    
            \label{fig:combined50}
        \end{subfigure}
        \centering
        \begin{subfigure}[b]{0.49\textwidth}
            \centering
            \includegraphics[scale=0.23]{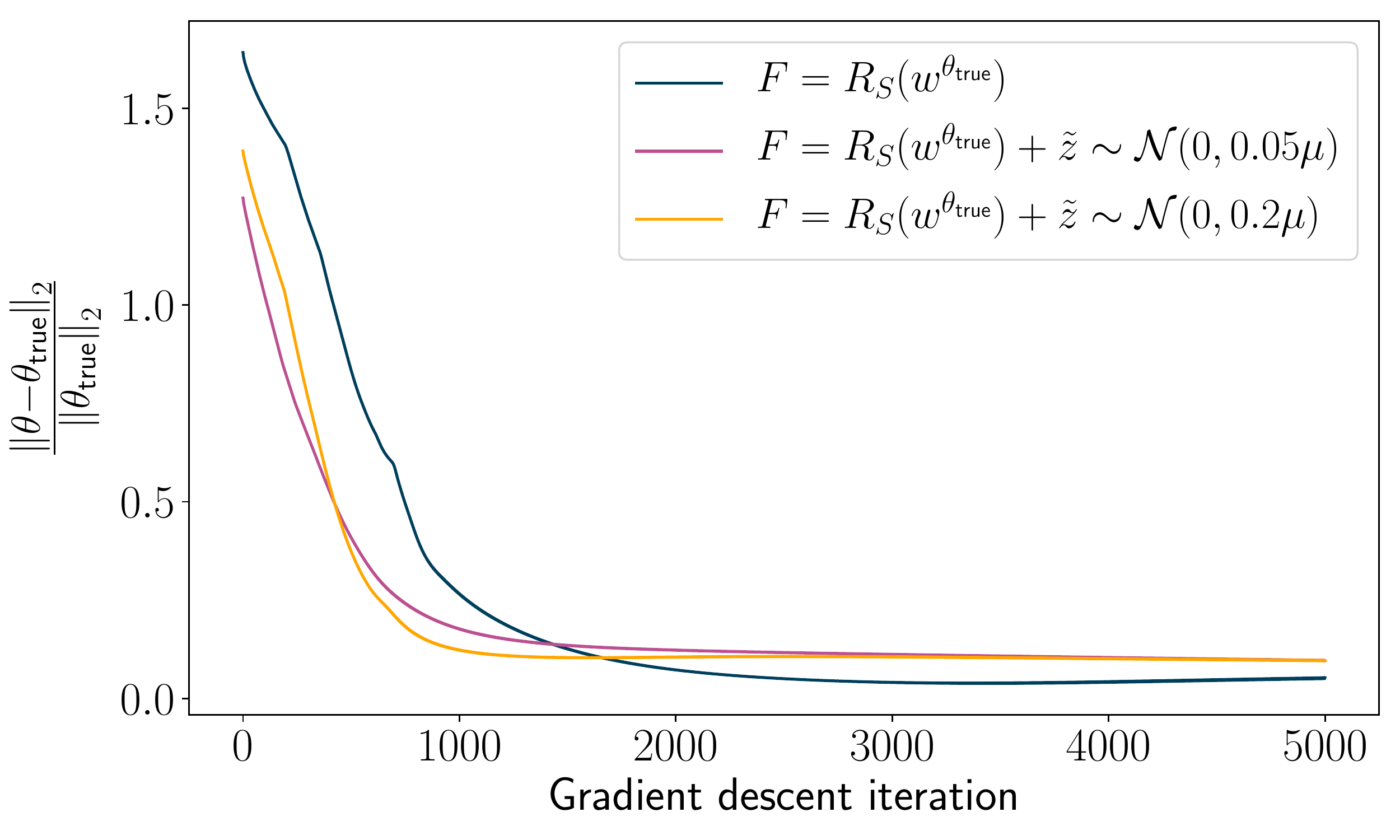}
            \caption[]%
            {{\small Landscape with continuous and discrete data with $N=100$.}}    
            \label{fig:combined100}
        \end{subfigure}
        \begin{subfigure}[b]{0.49\textwidth}  
            \centering 
            \includegraphics[scale=0.23]{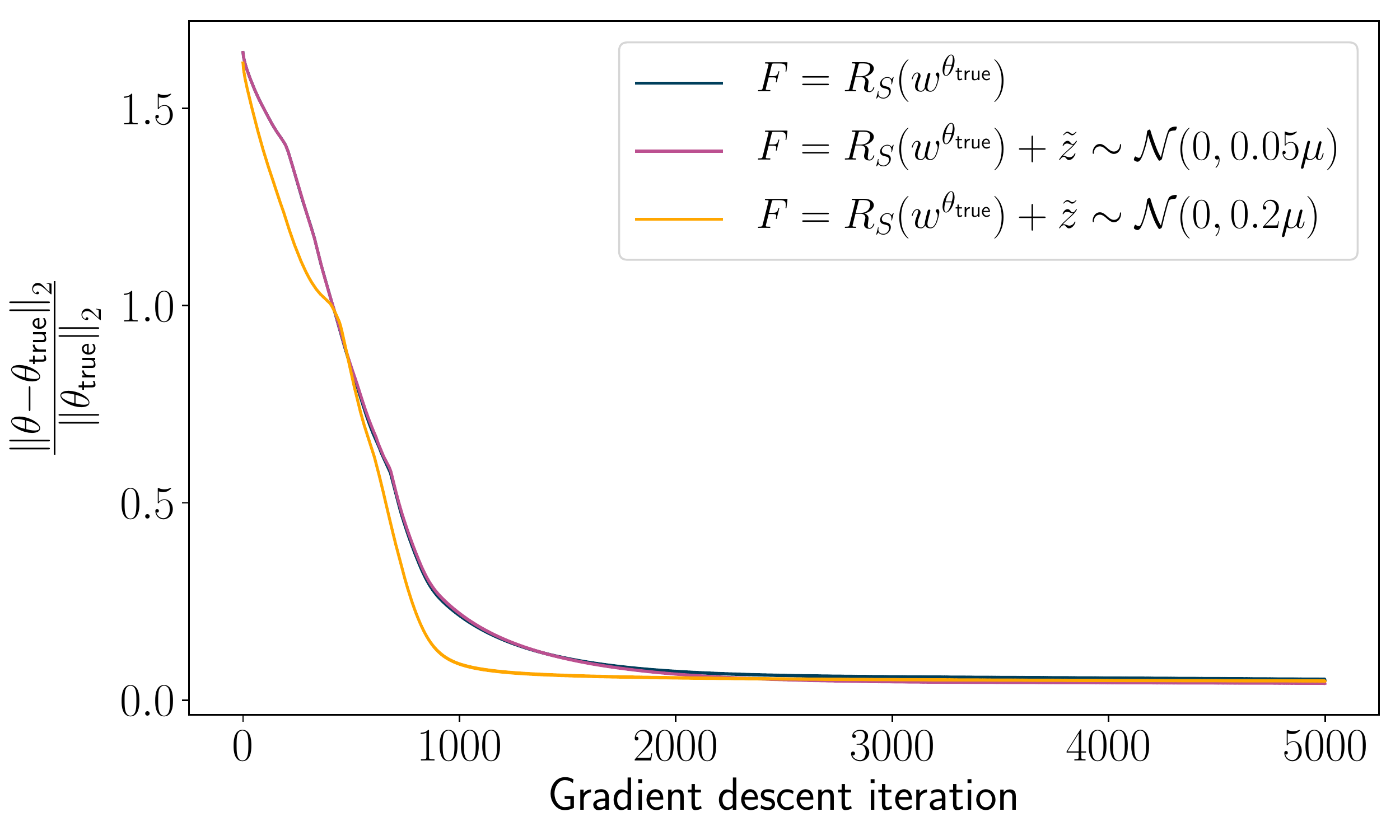}
            \caption[]%
            {{\small Landscape with continuous and discrete data with $N=150$.}}    
            \label{fig:combined150}
        \end{subfigure}
        \caption[]
        {\small Relative error between recovered parameters and true parameters with gradient descent iteration, for various values of noise standard deviation $\Tilde{\sigma}$, for landscape graph with continuous and discrete data. Higher value of $N$ provides better recovery.} 
        \label{fig:Combined}
\end{figure*}

\begin{figure*}[hbtp!]
        \centering
        \begin{subfigure}[b]{0.49\textwidth}
            \centering
            \includegraphics[scale=0.23]{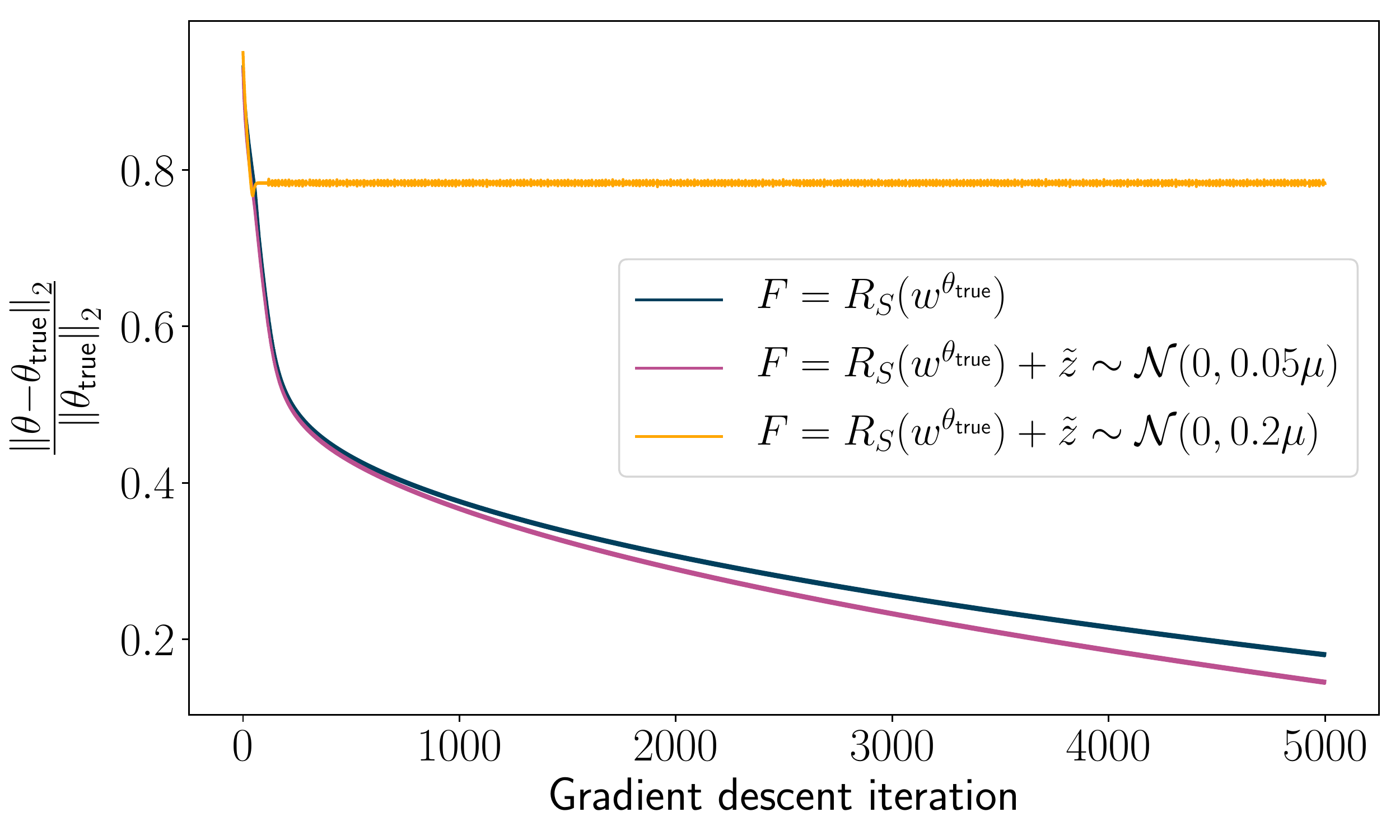}
            \caption[]%
            {{\small Landscape with continuous data and $N=25$.}}    
            \label{fig:beta25}
        \end{subfigure}
        \begin{subfigure}[b]{0.49\textwidth}  
            \centering 
            \includegraphics[scale=0.23]{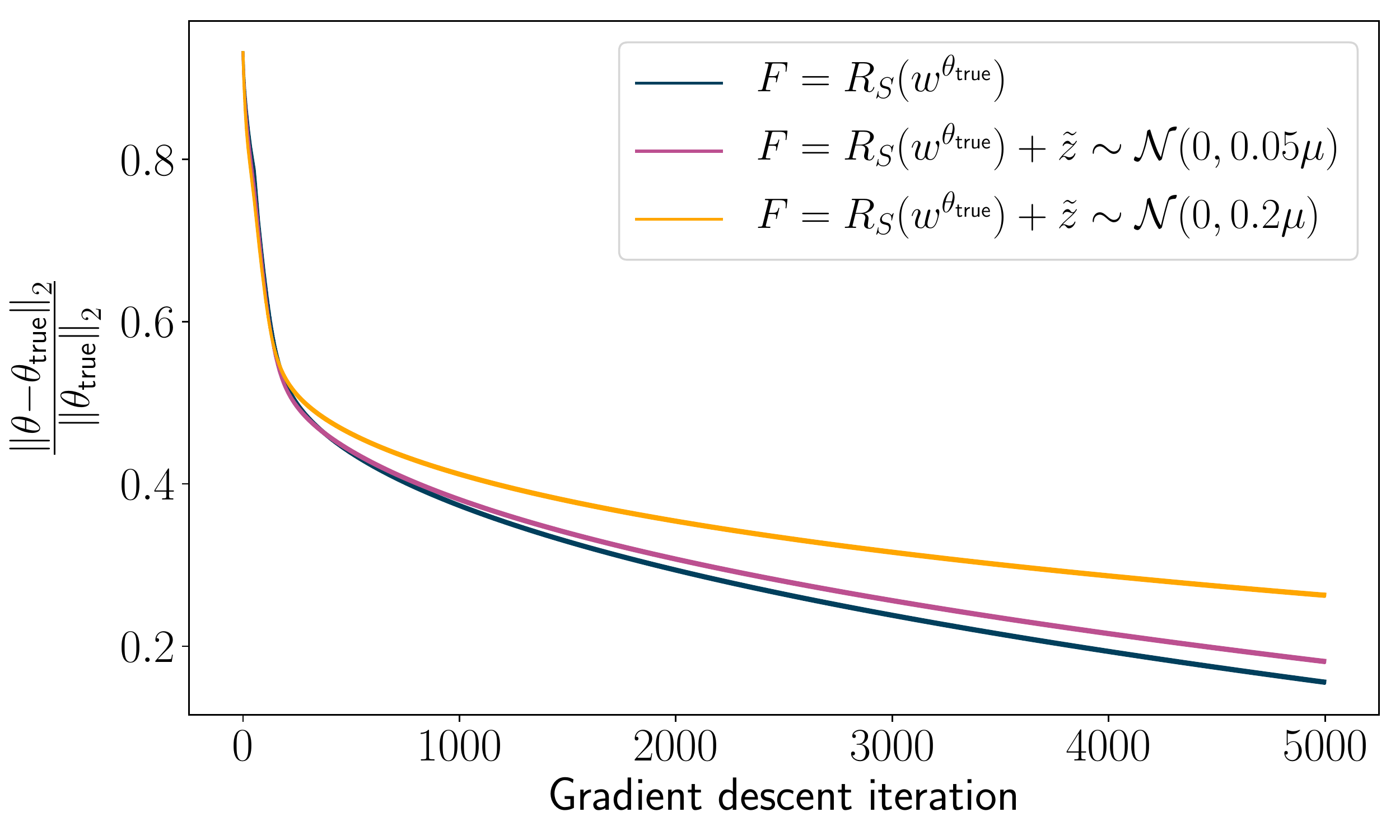}
            \caption[]%
            {{\small Landscape with continuous data and $N=50$.}}    
            \label{fig:beta50}
        \end{subfigure}
        \begin{subfigure}[b]{0.48\textwidth}   
            \centering 
            \includegraphics[scale=0.23]{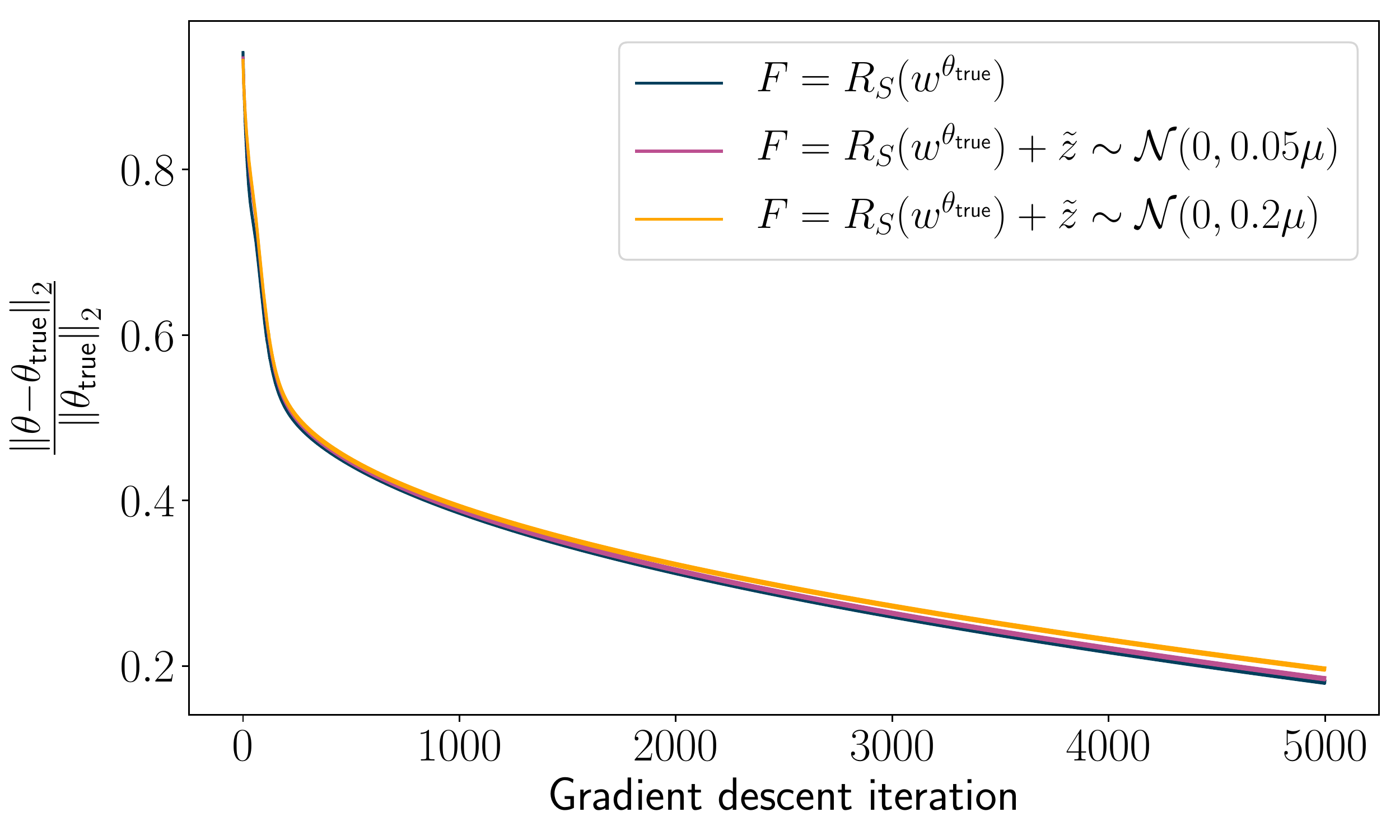}
            \caption[]%
            {{\small Landscape with continuous data and $N=100$.}}    
            \label{fig:beta100}
        \end{subfigure}
        \quad
        \begin{subfigure}[b]{0.48\textwidth}   
            \centering 
            \includegraphics[scale=0.23]{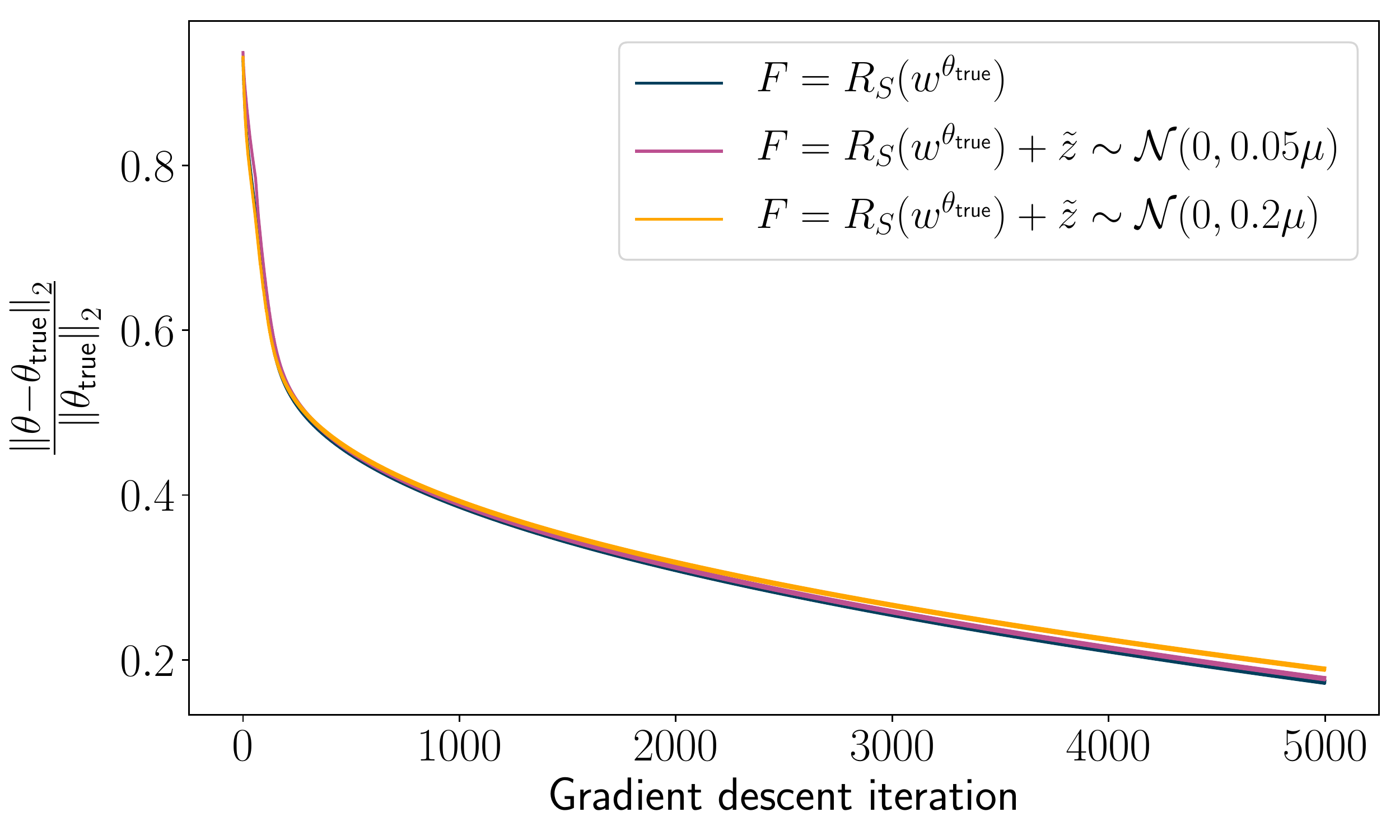}
            \caption[]%
            {{\small Landscape with continuous data and $N=150$.}}    
            \label{fig:beta150}
        \end{subfigure}
        \caption[]
        {\small Relative error between recovered parameters and true parameters for continuous data with gradient descent iteration, for various values of noise standard deviation $\Tilde{\sigma}$, for landscape graph with continuous data. Higher value of $N$ provides better recovery.} 
        \label{fig:Betas}
\end{figure*}

\begin{figure*}[hbtp!]
\parbox{\textwidth}{
        \begin{subfigure}{0.33\textwidth}
            \centering
            \includegraphics[scale=0.23]{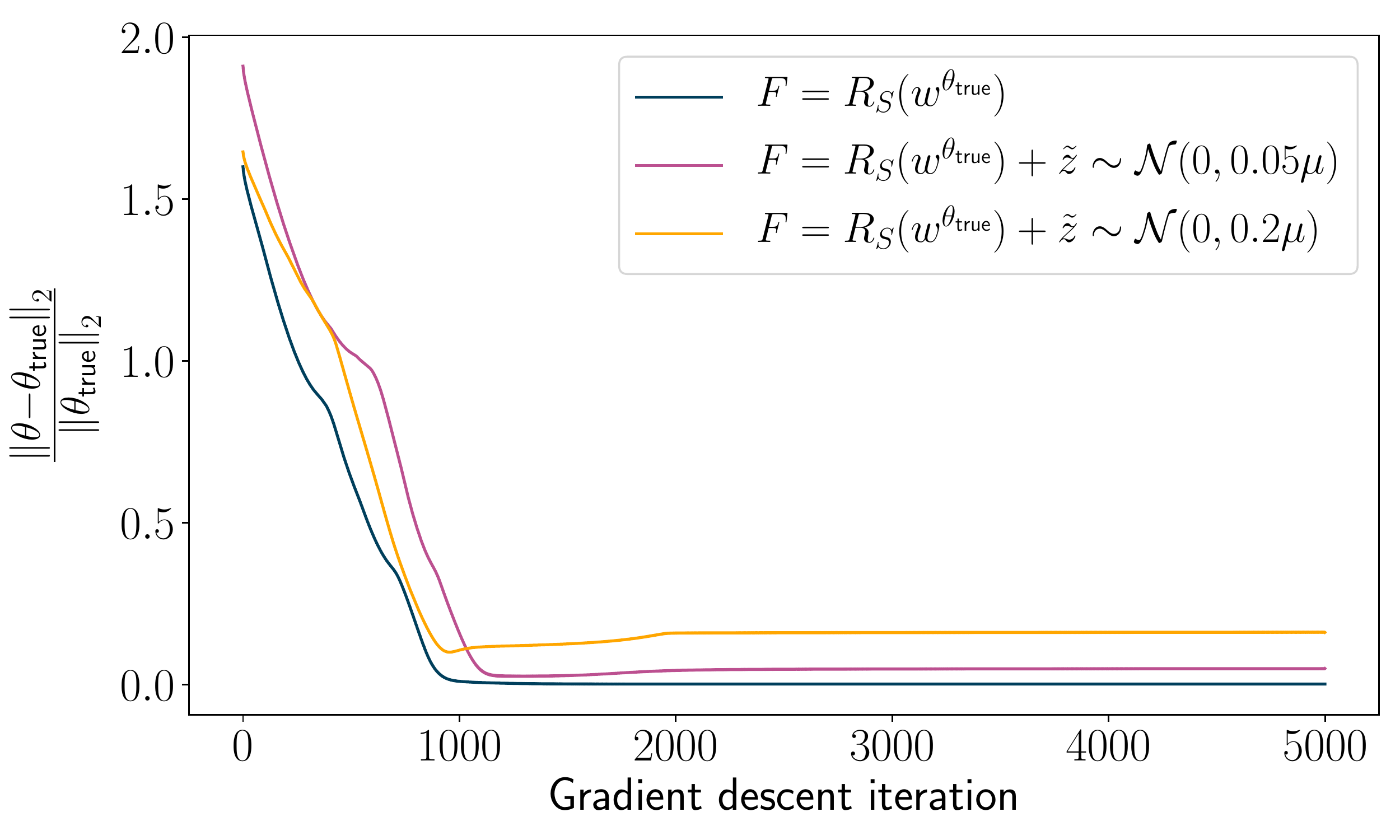}
            \caption[]%
            {{Landscape with discrete data and $N=25$.}}    
            \label{fig:alphas25}
        \end{subfigure}
        \begin{subfigure}{0.33\textwidth}  
            \centering 
            \includegraphics[scale=0.23]{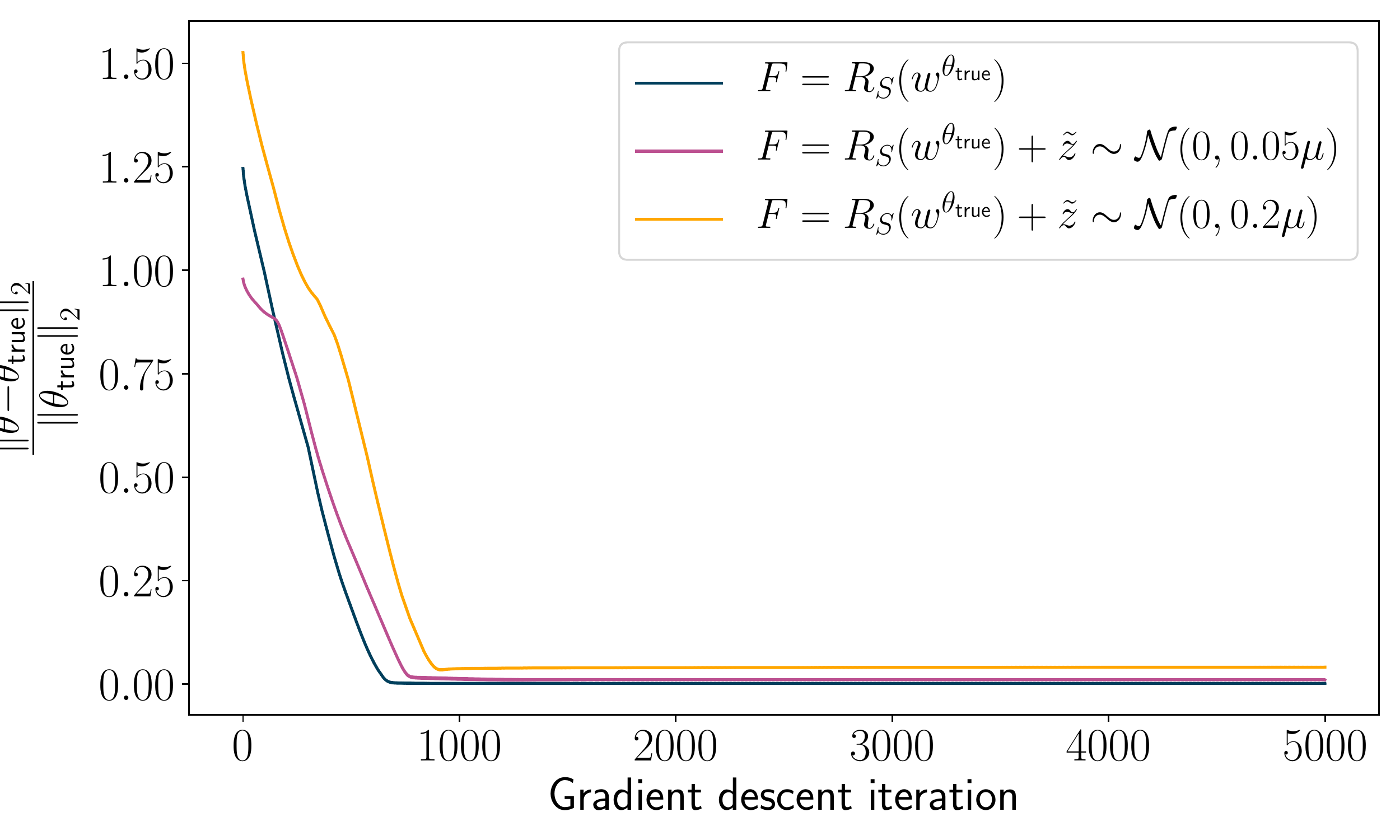}
            \caption[]%
            {{Landscape with discrete data and $N=50$.}}    
            \label{fig:alphas50}
        \end{subfigure}
        \begin{subfigure}[]{0.33\textwidth}   
            \centering 
            \includegraphics[scale=0.23]{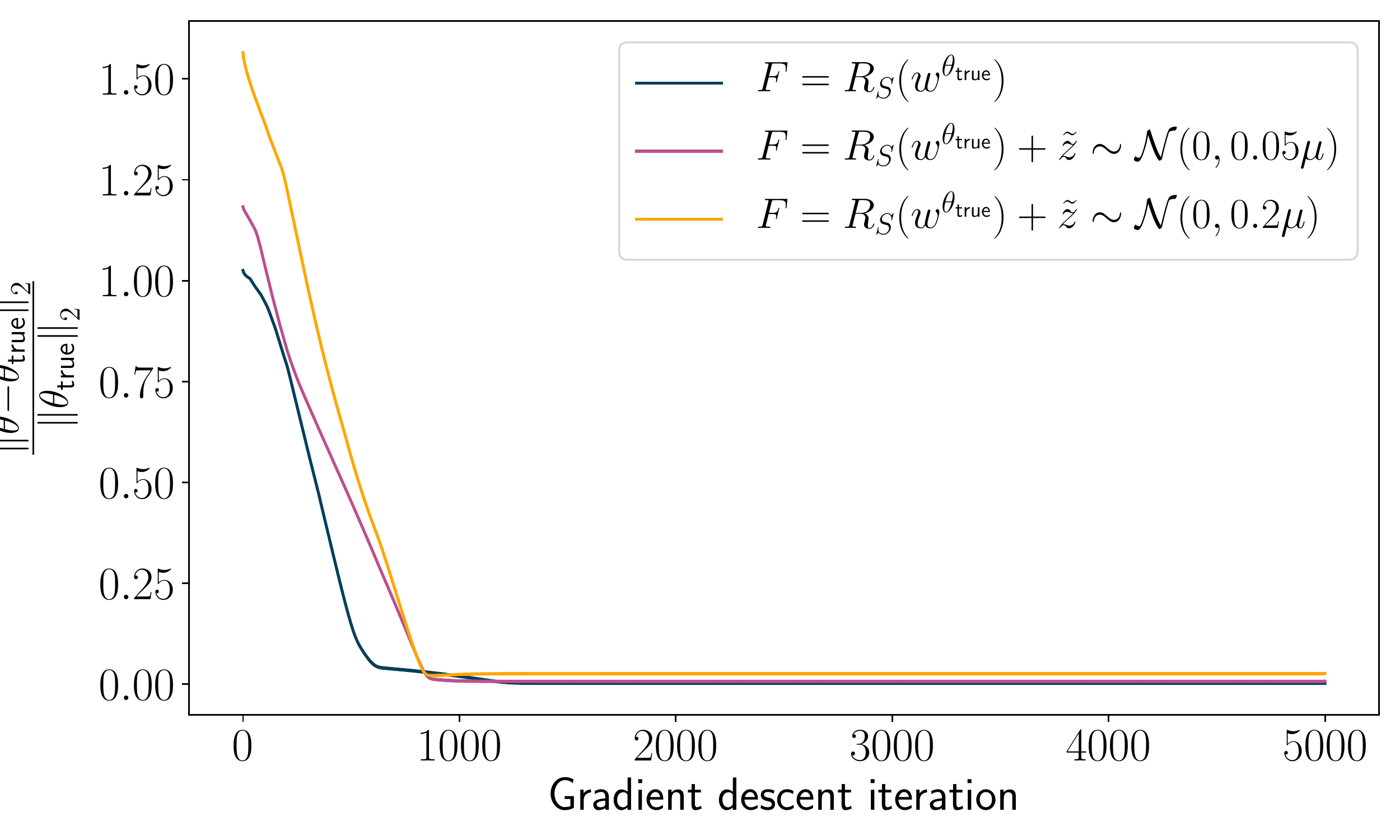}
            \caption[]%
            {{Landscape with discrete data and $N=100$.}}    
            \label{fig:alphas100}
        \end{subfigure}
    }
        \caption[]
        { Relative error between recovered parameters and true parameters for discrete data with gradient descent iteration, for various values of noise standard deviation $\Tilde{\sigma}$, for landscape graph with discrete data. Higher value of $N$ provides better recovery and we observe overfitting for high-noise with $N=25$.} 
        \label{fig:Alphas}
\end{figure*}

\subsection{Addressing overfitting}
For evaluating the generalizing ability of gradient-based optimization, we sample two sets of nodes as $S_{\text{train}}$ and $S_{\text{test}}$, learn parameters $\theta$ for $S_{\text{train}}$ and evaluate for pairwise effective resistance between nodes in $S_{\text{test}}$. As stated in section 4, with $|S_{\text{train}}|$ as low as 25 we obtain good generalization. Here, we present experiments with high value of $|S_{\text{train}}|$. The experiments were performed for landscape with continuous and discrete data. As the test loss converges, we obtain good generalization for parameters learnt with $N \geq 25$ as seen in Appendix Figure \ref{fig:Train_test}.

\begin{figure*}[hbtp!]
        \centering
        \begin{subfigure}[]{0.33\textwidth}
            \centering
            \includegraphics[scale=0.22]{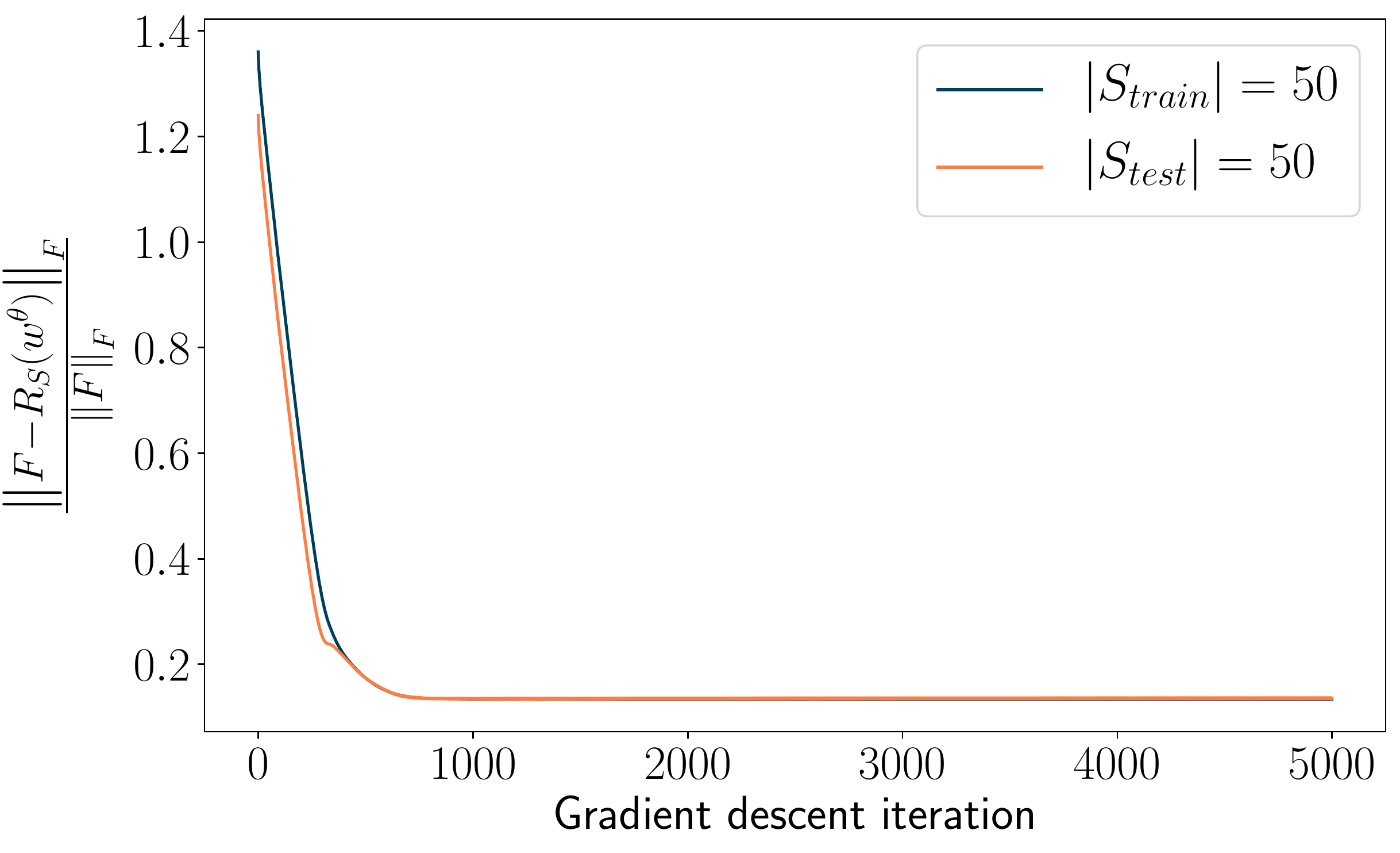}
            \caption[]%
            {{\small Relative train and test loss for $N_{\text{train}} = 50$.}}    
            \label{fig:train_test_50}
        \end{subfigure}
        \begin{subfigure}[]{0.33\textwidth}  
            \centering 
            \includegraphics[scale=0.22]{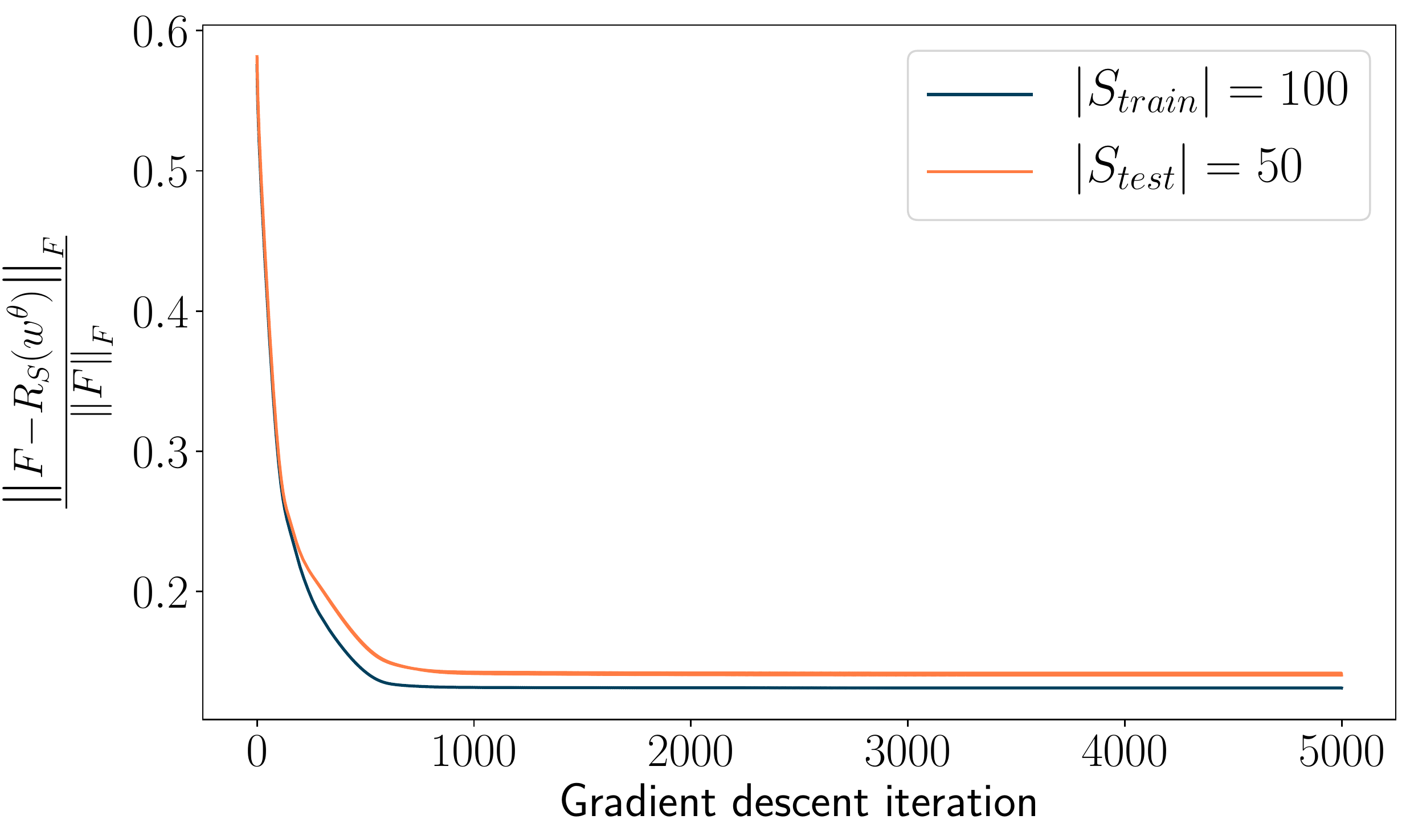}
            \caption[]%
            {{\small Relative train and test loss for $N_{\text{train}} = 100$.}}    
            \label{fig:train_test_100}
        \end{subfigure}
        \begin{subfigure}[]{0.33\textwidth}   
            \centering 
            \includegraphics[scale=0.22]{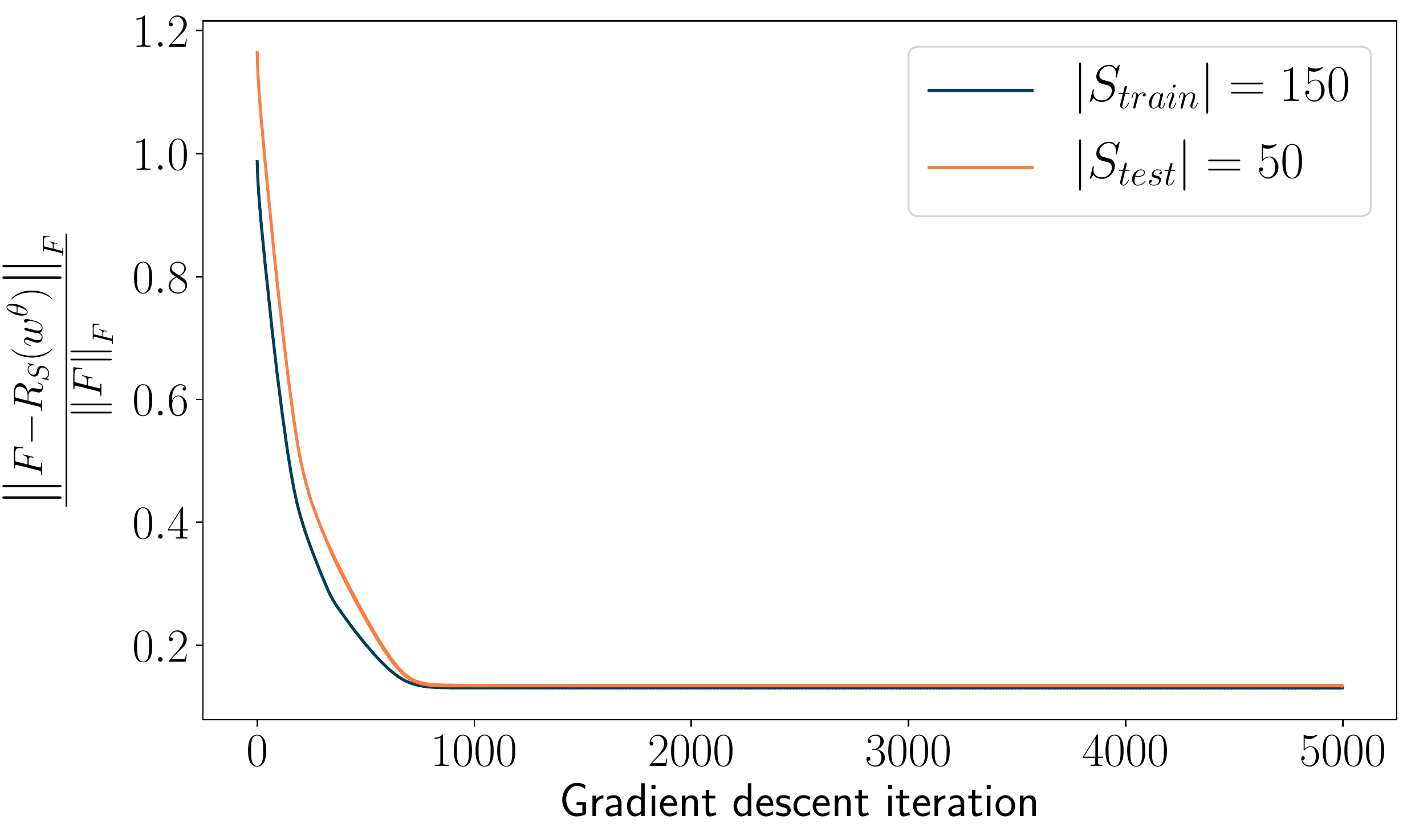}
            \caption[]%
            {{\small Relative train and test loss for $N_{\text{train}} = 150$.}}    
            \label{fig:train_test_150}
        \end{subfigure}
        \caption[]
        {\small Relative train and test loss for different values of sampled train nodes with $\Tilde{\sigma} = 0.2\mu$. We obtain good generalization, for $N \geq 25$, mitigating the concern for overfitting with low data.} 
        \label{fig:Train_test}
\end{figure*}

\begin{figure*}[hbtp!]
\parbox{\textwidth}{
\begin{subfigure}[]{.49\textwidth}
  \centering
  \includegraphics[scale=0.25]{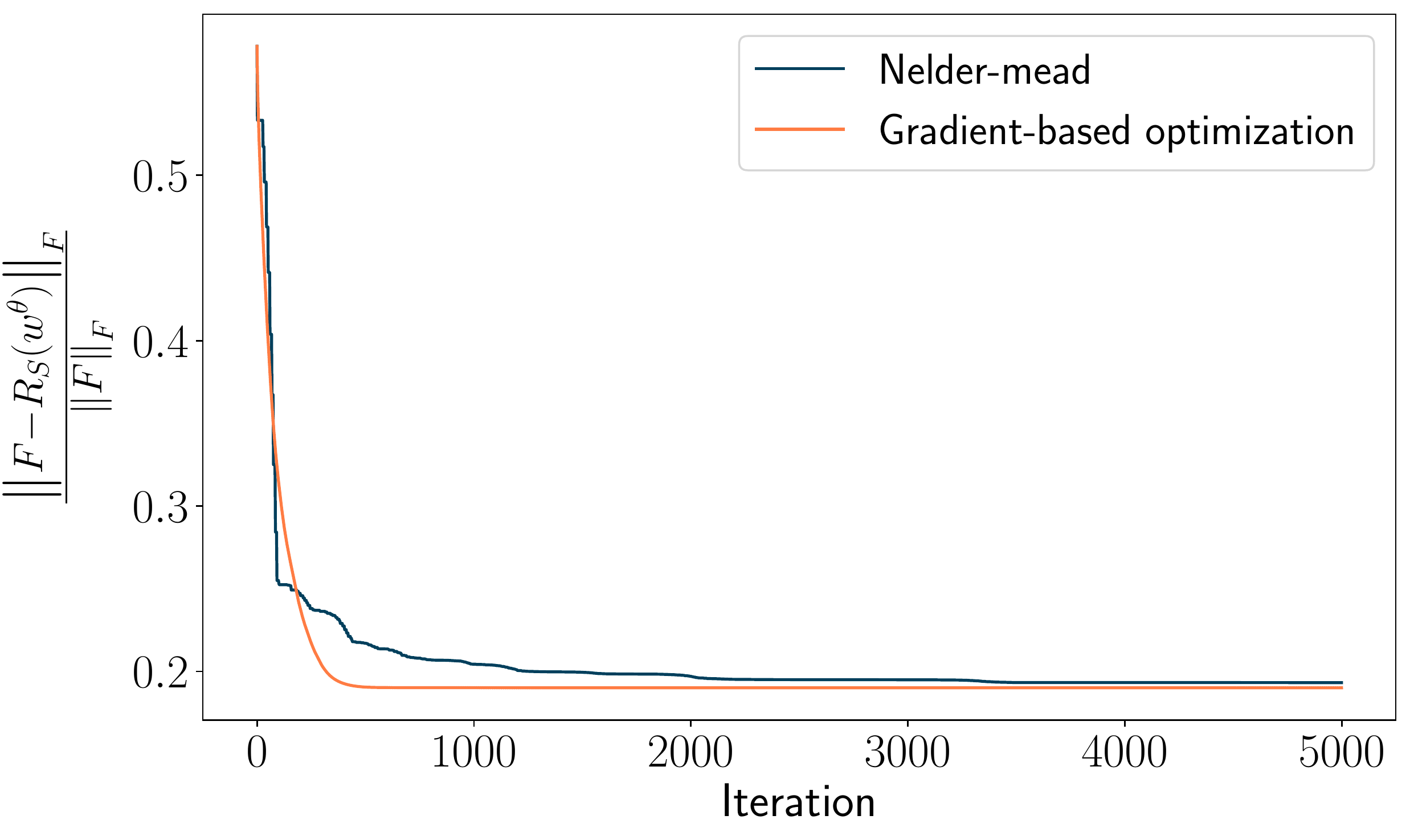}
  \caption{Comparison of relative loss with iteration for Nelder-mead and gradient-based approach.}
  \label{fig:nelder_mead_loss_noprj}
\end{subfigure}%
\begin{subfigure}[]{.49\textwidth}
  \centering
  \includegraphics[scale=0.25]{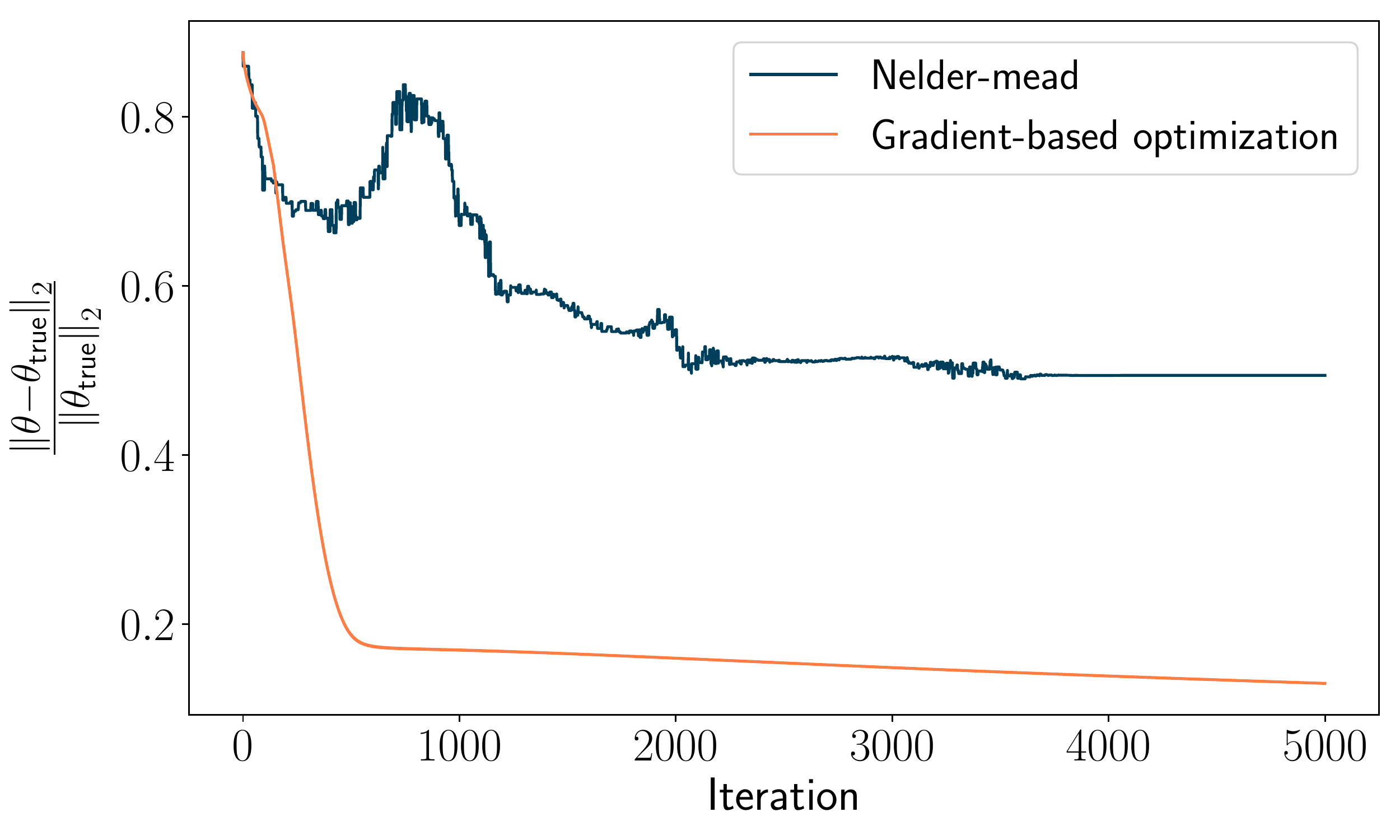}
  \caption{Relative parameter error between recovered parameters and true parameters with iteration.}
  \label{fig:nelder_mead_params_noprj}
\end{subfigure}
}
\caption{Comparison of proposed method to a heuristic optimization technique. Gradient-based optimization is faster in convergence and better at recovering true parameters with enough data. Experiments are for synthetic data with high noise setting with $N=150$ and $\Tilde{\sigma} = 0.2\mu$.
}
\label{fig:neldermead_noprj}
\end{figure*}

\subsection{Comparision with existing approaches}
We compare Nelder-mead \emph{without projection} and gradient-based optimization for learning the parameters for landscape graph. We observe that the gradient-based approach obtains better approximation to true parameters $\theta_{\text{true}}$ (refer Appendix Figure \ref{fig:neldermead_noprj}).

\section{Additional Related Work}
\label{sec:related_work}
Relevant related work on landscape genetics is discussed in Section 2. Here we add important comments on additionally related work in graph learning. In particular, in Section 3, we frame the inverse landscape genetics problem as a problem of learning edge weights in a graph from (noisy) measurements of the effective resistances. This problem was directly addressed in \citet{hoskins2018learning}, which our paper builds on. It has also been studied elsewhere. For example, it is well known that you can recover a graph exactly if you know effective resistances between \emph{all pairs of nodes}. 

This can be done in polynomial time \cite{spielmanLecture}: access to all effective resistances allows you to reconstruct the pseudoinverse of the graph Laplacian, which can then be inverted using a generic $O(n^3)$ time method, or more efficient algorithms \cite{jambulapati2018efficient}. Unfortunately, when only a subset of effective resistances are known, no polynomial time algorithm is known for recovering a graph consistent with those measurements. However, as observed in \citet{hoskins2018learning} and this paper (where we study a somewhat different parameterized problem) graph recovery can be framed as an optimization problem and solved to a global optimal with first order methods, despite inherent non-convexity. 
Recovering edge information from effective resistances has also been studied for the special case of tree graphs. In a tree, the effective resistance is the inverse of the shortest path distance between nodes $i$ and $j$. There has been a lot of interest in reconstructing trees from actual and partial measurements of these distances \cite{reyzin2007learning,Batagelj90}. Applications include the reconstruction of phylogenetic trees in genetics \cite{felsenstein1985confidence,felsenstein2004inferring}.

Finally, we note that our problem is related to that of inferring graphical models \cite{attias2000variational,mohan2012structured}, which has been studied in different formulations across machine learning, statistics, and graph signal processing \cite{egilmez2017graph,ortega2018graph}. The common assumption is that the correlation matrix between data at each node is related to the adjacency or Laplacian matrix of an unknown graph. Several work explore how many samples are needed to learn the structure of this graph, often under additional assumptions like graph sparsity \cite{Raskutti:2009,cai2011constrained}. Our work makes a structural assumption that the graph underlying our data has both a simple edge structure (i.e., its a grid graph) and that edges weights are functions of relatively low-dimensional edge data (i.e., landscape information). An interesting direction for future work is theoretically exploring the implications of these strong assumptions on bounding the sample complexity of the inverse landscape genetics problem.

\end{document}